\documentclass[preprint,12pt]{elsarticle}

\usepackage[utf8]{inputenc} 
\usepackage[T1]{fontenc}    
\usepackage{booktabs}       
\usepackage{amsfonts}       
\usepackage{nicefrac}       
\usepackage{microtype}      
\usepackage{xcolor}         
\usepackage{latexsym}
\usepackage{multirow}
\usepackage{makecell}
\usepackage{enumitem}
\usepackage{graphicx}
\usepackage{amsmath}
\usepackage{amssymb}
\usepackage{setspace}
\usepackage[algo2e]{algorithm2e}
\usepackage{algorithm}
\usepackage{algorithmicx}
\usepackage{algpseudocode}
\usepackage{subfigure}
\usepackage{latexsym}
\usepackage{inconsolata}
\usepackage{bm}
\usepackage{wrapfig}
\usepackage{hyperref}
\usepackage{url}
\usepackage{pifont} 

\usepackage{amsmath}
\usepackage{amssymb}
\usepackage{mathtools}
\usepackage{amsthm}
\usepackage{multirow}
\usepackage{pifont}
\usepackage{booktabs}
\usepackage{makecell}
\usepackage{bm}
\usepackage{colortbl}
\usepackage{wrapfig}
\usepackage{graphicx}
\usepackage{tablefootnote}

\usepackage{enumitem}
\setlist[itemize]{leftmargin=*, labelindent=0pt}

\theoremstyle{plain}
\newtheorem{theorem}{Theorem}[section]
\newtheorem{proposition}{Proposition}[section]
\newtheorem{lemma}{Lemma}[section]

\theoremstyle{definition}
\newtheorem{definition}[theorem]{Definition}

\theoremstyle{remark}

\newcommand{\cmark}{\ding{51}}%
\newcommand{\xmark}{\ding{55}}%

\colorlet{normalcolor}{white}
\colorlet{tablerowcolor}{gray!20} 
\newcommand{\belowcolormidrule}{
     \arrayrulecolor{normalcolor}
     \specialrule{\aboverulesep}{0pt}{0pt}%
     \arrayrulecolor{black}\specialrule{\lightrulewidth}{0pt}{0pt}%
     \arrayrulecolor{tablerowcolor}\specialrule{\belowrulesep}{0pt}{0pt}%
     \arrayrulecolor{black}}
\newcommand{\abovecolormidrule}{
     \arrayrulecolor{tablerowcolor}
     \specialrule{\aboverulesep}{0pt}{0pt}%
     \arrayrulecolor{black}\specialrule{\lightrulewidth}{0pt}{0pt}%
     \arrayrulecolor{normalcolor}\specialrule{\belowrulesep}{0pt}{-2.5pt}%
     \arrayrulecolor{black}
     }
\newcommand{\colorbottomrule}{\arrayrulecolor{tablerowcolor}\specialrule{\aboverulesep}{0pt}{0pt}%
            \arrayrulecolor{black}\specialrule{\heavyrulewidth}{0pt}{\belowbottomsep}}%

\journal{Neural Networks}

\begin{document}

\begin{frontmatter}

\title{Parallel Training in Spiking Neural Networks}

\author[1,2]{Yanbin Huang}
\ead{huangyanbin2023@ia.ac.cn}

\author[1]{Man Yao}

\author[1]{Yuqi Pan}

\author[3]{Changze Lv}

\author[1]{Siyuan Xu}

\author[3]{Xiaoqing Zheng}

\author[1]{Bo Xu}

\author[1]{Guoqi Li\texorpdfstring{\corref{cor1}}{}}
\ead{guoqi.li@ia.ac.cn}

\cortext[cor1]{Corresponding author}

\affiliation[1]{
organization={Institute of Automation},
addressline={Chinese Academy of Sciences},
city={Beijing},
country={China},
}
\affiliation[2]{
organization={School of Future Technology},
addressline={University of Chinese Academy of Sciences},
city={Beijing},
country={China},
}
\affiliation[3]{
organization={School of Computer Science},
addressline={Fudan University},
city={Shanghai},
country={China},
}

\begin{abstract}
The bio-inspired integrate–fire–reset mechanism of spiking neurons constitutes the foundation for efficient processing in Spiking Neural Networks (SNNs). Recent progress in large models demands that spiking neurons support highly parallel computation to scale efficiently on modern GPUs. This work proposes a novel functional perspective that provides general guidance for designing parallel spiking neurons. We argue that the reset mechanism, which induces complex temporal dependencies and hinders parallel training, should be removed. However, any such modification should satisfy two principles: \romannumeral1) preserving the functions of reset as a core biological mechanism; and \romannumeral2) enabling parallel training without sacrificing the serial inference ability of spiking neurons, which underpins their efficiency at test time. To this end, we identify the functions of the reset and analyze how to reconcile parallel training with serial inference, upon which we propose a dynamic decay spiking neuron. We conduct comprehensive testing of our method in terms of: \romannumeral1) Training efficiency and extrapolation capability. On 16k-length sequences, we achieve a 25.6× training speedup over the pioneering parallel spiking neuron \citep{fang2023parallel}, and our models trained on 2k-length can stably perform inference on sequences as long as 30k. \romannumeral2) Generality. We demonstrate the consistent effectiveness of the proposed method across five task categories (image classification, neuromorphic event processing, time-series forecasting, language modeling, and reinforcement learning), three network architectures (spiking CNN/Transformer/SSMs), and two spike activation modes (spike/integer activation). \romannumeral3) Energy consumption. The spiking firing of our neuron is lower than that of vanilla and existing parallel spiking neurons.

\end{abstract}

\end{frontmatter}

\section{Introduction}\label{sec:introduction}
Spiking neurons incorporate information across spatial and temporal domains into a membrane potential, i.e., the neuronal state. If this potential surpasses a threshold, the neuron fires a spike and the potential is reset; otherwise, it decays \cite{maass1997networks}. Thus, SNNs exhibit spike-based event-driven dynamics: sparse accumulations occur only upon spike transmissions between neurons, while the network stays idle otherwise \cite{roy2019towards}. Deploying SNNs on neuromorphic hardware \cite{merolla2014million,davies2018loihi,pei2019towards} yields significant energy savings. For example, the asynchronous sensing-computing neuromorphic chip Speck consumes merely 0.42 mW at idle, and its dynamic power under typical vision scenarios can be kept within the mW range \cite{yao2024spike}.

Directly training large-scale SNNs has long been a core challenge in the field. The progress can be viewed in three stages. \romannumeral1) Trainability under spike-based communication: surrogate‐gradient methods \cite{wu2018spatio,neftci2019surrogate} were proposed to handle the non-differentiable spike activation function, so that SNNs can be trained with backpropagation algorithm. \romannumeral2) Going deeper without performance loss: to reduce accuracy degradation in deeper SNNs, researchers introduced spiking residual connections \cite{fang2021deep,hu2024advancing}, novel architectures \cite{zhou2023spikformer,yao2023sdsa,yao2024spikedriven}, various normalization methods \cite{zheng2021going}, and training optimization methods \cite{NEURIPS2021_c4ca4238,NEURIPS2022_010c5ba0}. \romannumeral3) Efficient training under complex spatiotemporal dynamic constraints: the goal is to study how to efficiently train larger SNNs under longer sequences, laying the foundation for directly training large spiking models.

Regarding the challenge mentioned in the third stage above, the reset mechanism prevents parallel training, which makes SNN training very costly. One line of work keeps reset but speeds up training by decoupling spatial and temporal dependencies, for example by dropping temporal dependence during backpropagation \cite{xiao2022ottt, meng2023sltt}, by letting only a subset of neurons carry temporal information \cite{hu2024temporalreversible, xu2025pararevsnn}, or by using single-step pretraining followed by multi-step fine-tuning \cite{lin2024rethinking, yao2023attention}. Another line of work removes reset. PSN \cite{fang2023parallel} first took this direction and added a learnable parameter matrix along the time dimension to compensate for the role of reset. Some subsequent studies \cite{li2024parallel,su2024snn} have improved upon PSN, but the resulting models lose the serial inference property inherent to vanilla spiking neurons, which enables efficient computation and stable extrapolation potential at test time. Another idea is to remove reset and approximate the membrane potential of vanilla spiking neuron \cite{chen2025time,shen2025spikingssms,feng2025fpt}; this path is limited because the best possible performance cannot exceed that of the approximated neuron. 

Given the popularity of large foundation models \citep{bommasani2021opportunities,scaling_law}, an in-depth understanding of spiking neuron parallelization is imperative for SNNs aiming to scale alongside them. To this end, this work takes a novel functional perspective to analyze what constitutes a good design for parallel training in SNNs. \romannumeral1) We begin by focusing on the reset of vanilla spiking neurons, identifying its functions as introducing nonlinearity and controlling the membrane potential, which enhances the temporal dynamics of spiking neurons and enables the modulation of the membrane potential. We also highlight the drawbacks of the reset mechanism, including its inability to adequately fulfill the data-dependent adaptive membrane potential update and its hindrance in parallel training. \romannumeral2) We examine the general conditions under which spiking neurons support both parallel training and serial inference: their outputs depend solely on past inputs, allowing stepwise temporal iteration while supporting parallel computations across timesteps. These conditions serve as guiding principles for designing parallelized spiking neurons.

Based on these insights, we propose a dynamic decay spiking neuron with a causal convolution structure. Our approach can perform the functions of the reset mechanism more flexibly and thoroughly, while also supporting both parallel training and serial inference. We conduct comprehensive experiments to evaluate the proposed method, and the results demonstrate that our model outperforms existing parallel spiking neuron models in terms of training efficiency, extrapolation capability, generality, and energy efficiency. The key contributions of this work are as follows:

\begin{itemize}
\item \textbf{A Novel Functional View.} Parallel training in spiking neural networks is not merely about replacing the reset mechanism with a seemingly effective technique. Instead, it requires a systematic analysis of how the functions of reset are preserved or enhanced by the modification, and how parallel and serial modes can be made compatible under this change, which helps us understand the limitations of existing approaches.
\item \textbf{Design under an Insightful Guideline.} Inspired by the functional perspective, we propose a dynamic decay spiking neuron that implements functions better than reset, while supporting parallel training and remaining compatible with serial inference.
\item \textbf{Generality.} Our method demonstrates consistently competitive performance across various network architectures and tasks, while also exhibiting training efficiency, stable extrapolation, and energy benefits.
\end{itemize}

\section{Related Work}
\subsection{Spiking Neurons}
The transmission of electrical signals in biological neurons can be modeled with differential equations. Common spiking neuron models include Hodgkin-Huxley neurons \cite{hodgkin1952quantitative}, Leaky Integrate-and-Fire (LIF) neurons \cite{abbott1999lapicque}, Izhikevich neurons \cite{izhikevich2003simple}, etc. Among these, LIF neurons are the preferred choice for training deep SNNs due to their simplicity \cite{fang2021deep}. Currently, the two main techniques for addressing the non-differentiability issue in deep SNNs are converting an artificial neural network (ANN) into its SNN counterpart \cite{han2020rmp,bu2022optimal}, i.e., ANN-to-SNN, and direct training methods \cite{wu2018spatio, neftci2019surrogate, yao2021temporal,yao2023attention} which use surrogate gradients to implement backpropagation through time. In this paper, we focus on the latter approach. As spiking neural networks are being applied to more sequential processing tasks \citep{wang2024spikevoice, bal2024spikingbert, lv2024efficient}, how to strike a balance between constructing more complex spiking neurons and enabling efficient training has become an increasingly important concern.

\subsection{Parallel Training in SNNs}
Deep learning greatly benefits from large-scale parallel computing implemented on GPUs \citep{jeon2021deep}. To realize parallel training in SNNs, existing methods mainly use parallelizable modules to directly replace the reset mechanism or approximate the membrane potential of vanilla spiking neurons. For the former, PSN \cite{fang2023parallel} pioneers by introducing a learnable parameter matrix. In subsequent works, the alternatives focus primarily on the update method of membrane potential \cite{yarga2023accelerating, li2024parallel, su2024snn, xue2025multiplication} and the design of firing functions \cite{huang2024prf, chen2024pmsn, shen2025spikepack, bal2025p-spikessm}. However, most of them either abandon the inherent serial inference characteristics of vanilla spiking neurons or fail to fully preserve the functions of the reset mechanism. For the latter, the approximation methods for membrane potential range from a simple Bernoulli spike emission condition \cite{chen2025time}, a pre-trained surrogate dynamic network \cite{shen2025spikingssms}, to fixed-point iteration \cite{feng2025fpt}. Moreover, Spike-SSM \cite{zhong2024spike} deconstructs the membrane potential and iteratively resolves the output spikes through parallel max-min boundary compression. Overall, without modifying the intrinsic neuron structure, these methods do not surpass vanilla spiking neurons in performance.

\subsection{Decay Mechanism in Spiking Neuron} 
For SNNs, the decay factor, usually termed as the membrane time constant in LIF neurons, implies a limitation on expressiveness due to its fixed nature. PLIF \cite{fang2021incorporating} improves neuronal dynamics by making the decay factor learnable. Subsequent methods parameterize the decay factor via adjusting the parameter expression \cite{fang2023parallel, kosta2023adaptive, shi2023learnable, dan2025adaptive, zhang2025dalif}, integrating bidirectional parameters \cite{su2024snn}, introducing a complementary bypass \cite{huang2024clif}. In addition, some studies apply decay to the firing threshold \cite{yin2021accurate, bittar2022surrogate}. In dendritic neuron modeling, DH-LIF \citep{zheng2024temporal} learns multi-timescale dynamics by introducing heterogeneous decay factors across different dendritic branches. However, the decay factor remains static after training. In recent works, gating mechanisms \cite{yao2022glif,wang2024gated}, adaptive membrane time constant \cite{zhang2025balif} or self-connection circuit \cite{wang2024autaptic} have been employed to capture various biological features and enhance adaptiveness. What they have in common is that after training, the decay factor still changes with variations in input, membrane potential, and output spikes. This data-dependent paradigm inspires us to delve deeper into dynamic decay that is solely related to input.

\section{A Functional View of Parallelizing Spiking Neurons}\label{sec: functional view of parallelizing spiking neurons}

Removing the reset mechanism makes spiking neurons trainable in parallel. To understand what this change truly does, we need to answer two basic questions: \textbf{\romannumeral1) What is the function of reset; \romannumeral2) Under what conditions can spiking neurons simultaneously support parallel training and serial inference.} The first question helps us identify the functional deficiencies of prior work in reconstructing spiking neurons, thereby motivating strategies to compensate for that function, or even improve upon it. The second question is fundamental to the efficient inference of SNNs.


\subsection{Reset Mechanism and Its Function}

\textbf{Hard and Soft Reset.} In biological neurons, the depolarized membrane potential is restored to the resting state after the soma fires a spike \cite{luo2020principles}. Spiking neurons abstract the neuronal dynamics described above. Considering the trade-off between bio-plausibility and computational efficiency, the most widely used spiking neuron is the LIF, whose discrete iterative form is as follows \cite{wu2018spatio}:
\begin{align}
    H_t&=\beta V_{t-1}+(1-\beta)X_t, \label{eq:LIF_charge}\\
    S_t&=\Theta(H_t-V_\text{th}), \label{eq:LIF_fire}\\
    V_t&=
    \begin{cases}
    H_t(1-S_t)+V_{\text{reset}}S_t, &\text{hard reset}\\
    H_t-V_{\text{th}}S_t, &\text{soft reset}
    \end{cases}. \label{eq:LIF_reset}
\end{align}
\begin{wrapfigure}{r}{0.45\textwidth}
  \centering
  \vspace{-12pt}
  \includegraphics[width=0.45\textwidth]{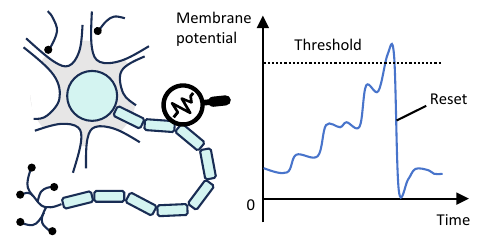}
  \vspace{-10pt}
  \caption{Illustration of a biological neuron (left) and the reset mechanism in neuronal dynamics (right).}
  \vspace{-12pt}
  \label{fig: biological neuron}
\end{wrapfigure}
In Eq. \ref{eq:LIF_charge}, the current input $X_t$ is integrated with the membrane potential $V_{t-1}$ from last timestep, and the decay factor $\beta=1-\frac{1}{\tau_m}$, where $\tau_m$ is membrane time constant. In Eq. \ref{eq:LIF_fire}, the Heaviside step function $\Theta(x)=1$ when $x\geq 0$, i.e. the membrane potential $H_t$ exceeds the threshold $V_\text{th}$, indicating that a spike is fired; otherwise, it is set to 0. According to how the membrane potential is regulated based on output spikes, reset can be generally categorized into hard and soft reset as depicted in Eq. \ref{eq:LIF_reset}. In hard reset, the charged membrane potential $H_t$ will be set to a constant $V_{\text{reset}}$ if a spike is fired, otherwise it will remain unchanged. $V_{\text{reset}}$ is commonly set to 0 for simplicity. In contrast, soft reset subtracts $H_t$ by $V_{\text{th}}$ when a spike is fired.

\textbf{Functions of Reset Mechanism.} The first function is to \textbf{introduce nonlinearity}. Specifically, the reset mechanism enriches the temporal dynamics of spiking neurons by establishing the following nonlinear relationship between the membrane potential and the input:
\begin{definition}
If the expression $H_t=g(X_1,X_2,...,X_t)$ is not a linear equation, the hidden state with respect to the inputs is considered nonlinear. \label{non-linearity}
\end{definition}
\noindent \textit{Remark}: Without reset, Eq. \ref{eq:LIF_charge} can be expanded into a linear form with respect to input.
\vspace{-5pt}
\begin{align}
    &H_t=\sum_{i=1}^t\beta^{t-i}(1-\beta)X_i.
\end{align}
In contrast, both hard and soft reset insert the firing function $f(.)$ into the iteration of membrane potential at two adjacent timesteps. Taking hard reset as an example, if letting $V_{\text{reset}}=0$ and combining Eq. \ref{eq:LIF_charge} and Eq. \ref{eq:LIF_reset}, we will derive one-step iteration form of the membrane potential. 
\begin{align}
    &H_t=\beta(1-f(H_{t-1}))H_{t-1}+(1-\beta)X_t.\label{eq:LIF_iter}
\end{align}
Obviously, it cannot be transformed into an input-dependent linear equation. This is similar for soft reset as well. Therefore, we conclude that reset introduces nonlinearity, and several parallel spiking neuron designs \cite{fang2023parallel, yarga2023accelerating, bal2025p-spikessm} ignore this role.

The second function is to \textbf{control membrane potential.} The reset mechanism constrains the membrane potential within a suitable range and averts ceaseless spike firing. For clarity, we quantitatively describe the control ability over the membrane potential as $\Delta$-short control and long control:

\begin{figure*}[t]
\vskip 0.2in
\vspace{-12pt}
\begin{center}
\centerline{\includegraphics[width=\columnwidth]{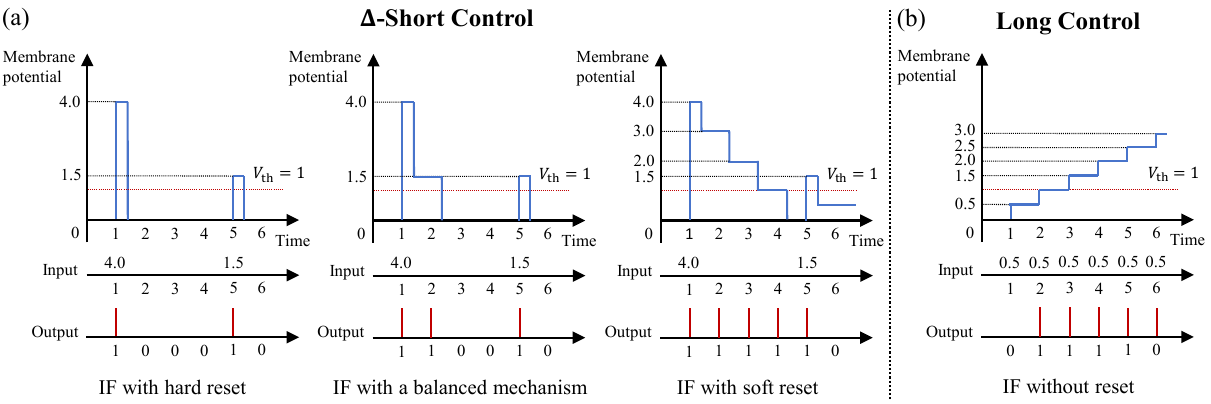}}
\vspace{-5pt}
\caption{The reset mechanism serves the function of regulating the membrane potential in an input-dependent manner, which can be categorized into $\Delta$-short and long control. (a) 
$\Delta$-short control. \emph{Left:} Hard reset enforces short control at $\Delta=1$ level, but it does not allow spatial discriminability between inputs of varying importance. Regardless of how large the membrane potential is at the current timestep, it will be forcibly reset to zero if it exceeds the threshold. \emph{Right:} Soft reset extends the control duration as the input magnitude increases, which can lead to continuous spike firing and reduced temporal discriminability. That is, compared to hard reset, soft reset allows the input at the current timestep to influence several subsequent timesteps; but it also introduces the challenge of making it difficult to distinguish which inputs across different timesteps are more important. \emph{Middle:} We therefore seek a balanced mechanism that adaptively determines the duration of membrane potential influence based on the input. (b) Long control. Without reset, even under a relatively small constant input sequence, e.g., \{0.5\}, the membrane potential of a spiking neuron would continuously accumulate, leading to infinite spike firing.}\label{fig: membrane_control}
\end{center}
\vskip -0.2in
\end{figure*}

\begin{definition}
There exists an $\Delta \in \mathbb N^+$ such that, for any $t>\Delta$, if $H_{t-\Delta}\geq V_{\text{th}}$  and $X_{t-\Delta+1}$, ..., $X_t < V_{\text{th}}/\Delta$, it always holds that $H_t<V_{\text{th}}$. In this case, the spiking neuron is said to have \textbf{$\Delta$-short control} over the membrane potential. \label{local_control}
\end{definition}
\noindent \textit{Remark}: The reset mechanism controls how long a large membrane potential affects the spiking neuron. For example, consider an IF neuron without a decay factor, with $V_{\text{th}}=1$ and $H_1=X_1=4$. In hard reset, the membrane potential is immediately set to 0 after a spike firing, so the effect of the large input lasts for $\Delta=1$ timestep. In contrast, with a soft reset, the spike persists for 4 timesteps. $\Delta$-short control ensures that a very large input affects the spiking neuron only within a relatively short time window $\Delta$, thereby preventing prolonged spike firing (see Fig. \ref{fig: membrane_control}).
\begin{definition}
If the input sequence $\{X_t\}$ has an upper bound $C$, then the membrane potential sequence $\{H_t\}$ also has an upper bound $C_H$. In this case, the spiking neuron is said to have \textbf{long control} over the membrane potential. \label{global_control}
\end{definition}
\vspace{-8pt}
\noindent \textit{Remark}: Long control prevents sustained spike firing or even membrane potential explosion caused by small inputs that accumulate without being reset (see Fig. \ref{fig: membrane_control}). In other words, long control keeps the membrane potential stable over an arbitrarily long period of time.


\textbf{Towards Better Functional Realization.} Although we have identified two functions of reset, the reset mechanism itself is not the optimal realization of these functions. Specifically, the nonlinear response is binary—either 0 or 1—lacking diversity. More importantly, the control of membrane potential via hard or soft reset is highly rigid and lacks flexibility. In hard reset, regardless of how large the membrane potential becomes, its influence is terminated after a single timestep. This strict mechanism severely limits spatial discriminability, as spike generation cannot reflect differences in input magnitude beyond threshold. Conversely, soft reset subtracts a fixed amount from the membrane potential upon spiking. For large inputs, multiple timesteps are required to neutralize their accumulated effect, leading to prolonged spike firing. As a result, spike events become increasingly dependent on past inputs, degrading temporal discriminability and obscuring the contribution of the current input.

Therefore, in many existing parallel spiking neuron designs, structures that attempt to approximate either hard reset \cite{shen2025spikingssms, feng2025fpt} or soft reset \cite{li2024parallel, huang2024prf, chen2024pmsn} can at best reproduce functions similar to those of the reset mechanism, but cannot achieve functions beyond it. Recognizing the inherent limitations of reset helps us focus on enhancing the two functions abstracted from it.

\subsection{Conditions for supporting parallel training and serial inference} \label{sec: condition}
A natural idea to realize parallel training in SNNs is to replace the reset mechanism with other parallelizable technique. However, some previous works \cite{fang2023parallel, li2024parallel, su2024snn} sacrifice the inherent efficiency of serial inference in vanilla spiking neurons, where the membrane potential at each timestep can be computed solely from the preceding membrane potential (or a small fixed set of states) and the current input. Consequently, such approaches introduce increased computational and memory overhead during inference and may even fail to generalize beyond the training sequence length.

We therefore argue that parallel training in SNNs should remain compatible with serial inference, allowing appropriate computational modes at different stages. Through observation and induction, we further identify three structural conditions that spiking neurons must satisfy to enable parallel training while retaining efficient serial inference.

\begin{figure*}[t]
\vskip 0.2in
\vspace{-12pt}
\begin{center}
\centerline{\includegraphics[width=\columnwidth]{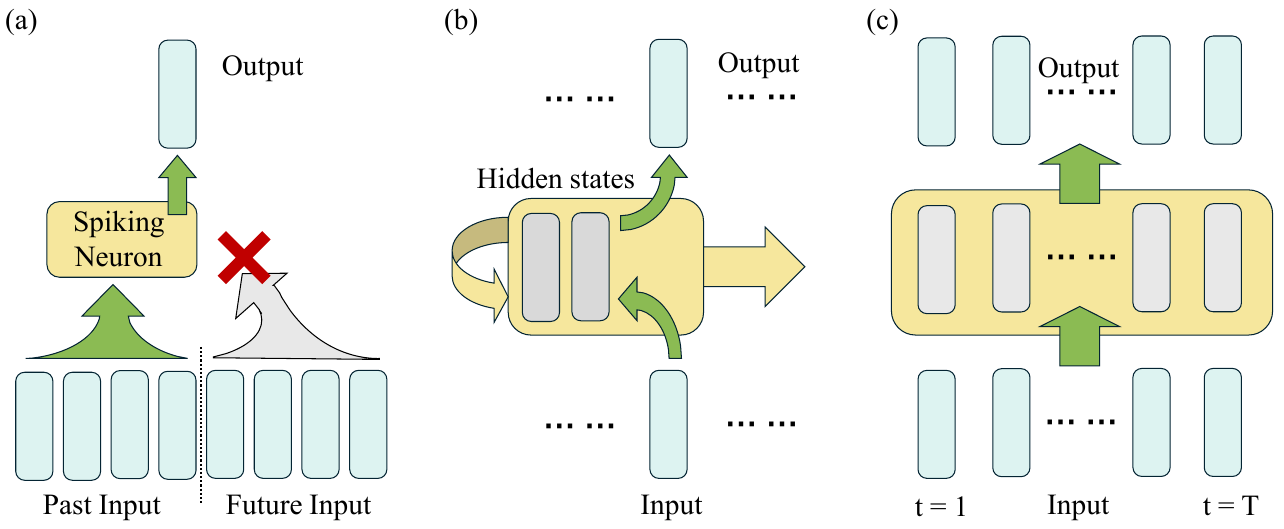}}
\vspace{-5pt}
\caption{Illustration of the three conditions for spiking neurons to achieve parallel training and serial inference. (a) Condition 1: Prefix summarizability. The output can only be determined by inputs from the past (the prefix), and this dependency does not change over time. (b) Condition 2: Online updatability. The internal hidden states can be recurrently updated as the sequential input arrives. (c) Condition 3: Offline parallelizability. For a fixed-length sequence, the output can be computed via parallel computation through time.}\label{fig: condition}
\end{center}
\vskip -0.2in
\end{figure*}


\textbf{Condition 1: Prefix Summarizability.} At any time step $t$, the output $S_t$ of the spiking neuron depends only on a representation $H_t$ determined by the prefix $X_{1:t}$ ($X_1,X_2,...,X_t$).
\begin{align}
    \forall t\ge 1, \;H_t=\phi_{\theta}(X_{1:t}), \;S_t=g_{\theta'}(H_t)
\end{align}

\noindent \textit{Remark}: This condition implicitly requires a specific structural relationship between the intermediate representation $H$ and the input $X$. First, $H_t$ must be causal, depending only on the prefix $X_{1:t}$ and not on any future inputs. Second, the prefix summary $\phi$ must be time-invariant, with parameters $\theta$ shared across all time steps.

\textbf{Condition 2: Online Updatability.} The prefix representation $H_t$ can be updated online through a finite amount of computation during inference.
\begin{align}
    \forall t\ge 1, \; H_t=u_{\theta}(H_{t-1},X_t) \label{eq:online_updatability}
\end{align}

\noindent \textit{Remark}: This condition indicates that the membrane potential update process must reuse computation results from previous time steps.

\textbf{Condition 3: Offline Parallelizability.} All $s_t$ within a window of size $T$ can be obtained by parallel computation through time during training.
\begin{align}
    \exists p, \;\forall t\in [1, T], \;  \text{s.t. }H_t=p_{\theta}(X_{1:t})
\end{align}

\noindent \textit{Remark}: This condition suggests that, given the offline sequence $x_{1:T}$, there exists a computational graph (e.g., convolution, matrix multiplication, or parallel scan) for $H_{1:T}$ along the temporal dimension that is independent of the recursive execution order in Eq. \ref{eq:online_updatability}.

\begin{table*}[t]
\vspace{-4pt}
\caption{A review of vanilla spiking neurons and its typical parallel variants. Function 1: Introducing nonlinearity. Function 2: Controlling the membrane potential. PTSI: Parallel training and serial inference. Condition 1: Prefix summarizability. Condition 2: Online updatability. Condition 3: Offline parallelizability.}
\label{tab: conditions}
\vspace{4pt}
\centering
\renewcommand{\arraystretch}{1.15}
\resizebox{\linewidth}{!}{
\setlength{\tabcolsep}{3pt}
\begin{tabular}{c|cc|ccc}
\toprule
\multirow{2}{*}[-0.5ex]{Methods} & \multicolumn{2}{c|}{Functions of Reset} & \multicolumn{3}{c}{Conditions of PTSI} \\ \cmidrule{2-6}
    &   Function 1 & Function 2  & Condition 1   &Condition 2 &    Condition 3     \cr\midrule
LIF \citep{abbott1999lapicque} & \cmark & \cmark & \cmark & \cmark & \xmark \cr
PSN \citep{fang2023parallel} & \xmark & \xmark & \xmark & \xmark & \cmark \cr
Masked PSN \citep{fang2023parallel} & \xmark & \xmark & \xmark & \xmark & \cmark \cr
Sliding PSN \citep{fang2023parallel} & \xmark & \cmark & \cmark & \cmark & \cmark \cr
PRF \citep{huang2024prf} & \xmark & \cmark & \cmark & \cmark & \cmark \cr
IPSU \citep{li2024parallel} & \cmark & \xmark & \xmark & \xmark & \cmark \cr
BPSN \citep{su2024snn} & \xmark & \xmark & \xmark & \xmark & \cmark \cr \midrule
\textbf{DSN (Ours)} & \cmark & \cmark & \cmark & \cmark & \cmark \cr
\bottomrule
\end{tabular}
}
\end{table*}

As the conclusion of this section, Table \ref{tab: conditions} summarizes the extent to which various spiking neurons satisfy the functions of reset and the conditions required for enabling parallel training while remaining compatible with serial inference. It is worth noting that the early and well-known PSN and Masked PSN \citep{fang2023parallel} have learnable parameters coupled with the training timestep, and thus fail to satisfy Conditions 1 and 2. More critically, PSN cannot guarantee causality in sequential computations, and its number of parameters grows quadratically with the sequence length, thereby incurring substantial computational overhead. Therefore, they are not ideal candidates for parallelizable spiking neurons. Instead, the parallelizable spiking neuron we seek should not only preserve—or even enhance—the two functions of reset, but also satisfy the three conditions required for parallel training and serial inference.

\section{Methods: Dynamic Decay Spiking Neuron}\label{sec: methods}

\begin{figure*}[t]
\vskip 0.2in
\vspace{-12pt}
\begin{center}
\centerline{\includegraphics[width=\columnwidth]{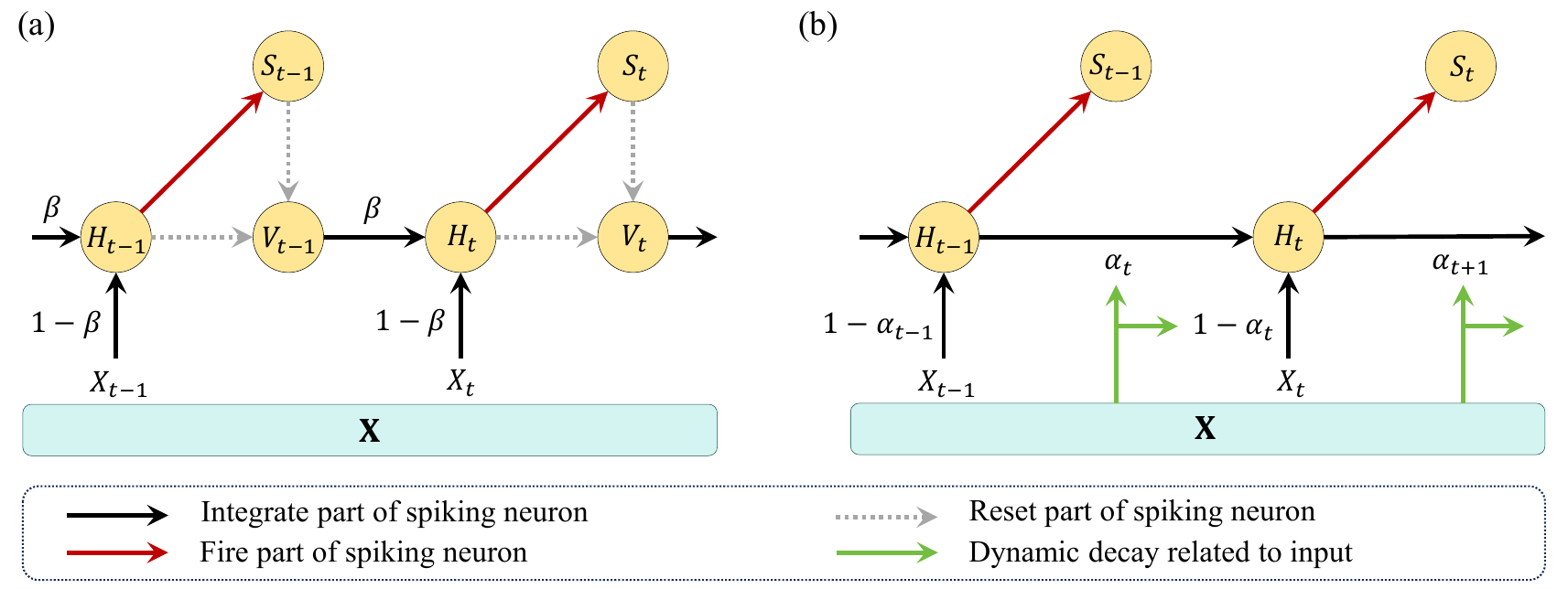}}
\vspace{-5pt}
\caption{Illustration of the computational process of LIF spiking neuron and reset-free spiking neuron with dynamic decay. (a) LIF neuron. The current input and the membrane potential from last timestep are integrated with a constant decay factor $\beta$. The integrated $H_t$ determines the firing of the spike $S_t$, which in turn decides whether $H_t$ is reset. (b) Dynamic decay. After replacing the reset mechanism with dynamic decay $\alpha_t$, the membrane potential can be computed both serially and in parallel.}
\end{center}
\vskip -0.2in
\end{figure*}

Sec. \ref{sec: functional view of parallelizing spiking neurons} provides general guidance for designing parallelizable spiking neurons through the lens of functional analysis. Building on this, we propose a Dynamic Decay Spiking Neuron \textbf{(DSN)}, which includes two modifications to vanilla spiking neurons: i) the reset mechanism is removed and the constant decay $\beta$  is replaced with a dynamic decay $\alpha_t$. Here, $\alpha_t$ is obtained via a causal convolution over the input. ii) the spike firing pattern is optimized by incorporating emerging integer-valued training techniques \cite{luo2025integer}. DSN has the following vectorized serial form:
\begin{align}
    \mathbf{H}_t&=\bm{\alpha}_t \odot\mathbf{H}_{t-1}+(1-\bm{\alpha}_t)\odot\mathbf{X}_t.\label{eq:dsn_charge}\\
    \mathbf{S}_t&=\mathrm{Clip}[\mathrm{Round}(\mathbf{H}_t),0,N].\label{eq:dsn_fire}
\end{align}
Here, the input $\mathbf{X}_t \in \mathbb{R}^{C\times1}$ has $C$ channels. $\odot$ denotes element-wise product. $\mathrm{Round(\cdot)}$ indicates rounding to the nearest integer. $\mathrm{Clip}[x,0,N]$ means clipping the input $x$ to the range [0, $N$]. $N$ is a positive integer, representing the upper limit of the number of spikes to be emitted. 

In Eq. \ref{eq:dsn_charge}, we derive the dynamic decay $\bm{\alpha}_t$ from $\mathbf{X}_{t-k+1:t}$ ($k$ inputs from $\mathbf{X}_{t-k+1}$ to $\mathbf{X}_{t}$) as follows:
\begin{align}
    &\bm{\alpha}_t'=\mathrm{CausalConv1D}(\mathbf{X}_{t-k+1:t}),\label{eq:dsn_causalconv}\\
    &\bm{\alpha}_t=\mathrm{Sigmoid}(\bm{\alpha}_t')^{1/\tau}\label{eq:dsn_ffn}
\end{align}
Here, $\mathrm{CausalConv1D(\cdot)}$ is a causal 1D convolution to mix the features from the past $k$ inputs. $\mathrm{Sigmoid}$ function is chosen to constrain $\bm{\alpha}_t$ between 0 and 1. $\tau$ is a hyperparameter to fine-tune $\bm{\alpha}_t$. 

\textbf{Design Rationale.} After removing the reset mechanism, we find that a varying decay factor can also introduce nonlinearity and control membrane potential, thereby restoring the functions of reset. This forms the basis of our initial design. The causal convolution is usually short but has been shown to be effective in capturing short-term dependency \cite{gu2024mamba, de2024griffin}. The optimized spike firing pattern helps reduce training overhead and learn better representations \cite{yao2025scaling}. Moreover, we can choose to introduce an extra learnable parameter $\mathbf{W} \in \mathbb{R}^{C\times C}$ to mix the features across different channels of $\bm{\alpha}_t'$ before applying the $\mathrm{Sigmoid}$ function in Eq. \ref{eq:dsn_ffn}, i.e. $\mathbf{W}\bm{\alpha}_t'$. This enhanced DSN is suitable as a complete block to further improve the modeling ability of SNNs.

\textbf{Functions Superior to Reset.} DSN is a specific implementation of dynamic decay $\alpha_t$, which can be related to input at preceding timesteps and is usually limited to between 0 and 1 using a non-linear activation function with learnable parameters $\theta$:
\begin{align}
    &\alpha_t=\phi_{\theta}(X_t, X_{t-1}, ...)\in[0, 1]\label{eq: dynamic_decay_general_form}
\end{align}
In fact, we can prove that dynamic decay in Eq. \ref{eq: dynamic_decay_general_form} can implement all the functions of reset.

\begin{proposition}
Dynamic decay can introduce nonlinearity and enabling more flexible $\Delta$-short and long control of the membrane potential than the reset mechanism.\label{equivalent_function_reset_dynamic_decay}
\end{proposition}
\noindent \textit{Remark}: We provide the detailed proof in \ref{detail_dynamic_decay_proof}. An intuitive interpretation is that the variability of $\alpha_t$ broadens the expressive range of nonlinearity and allows adaptive control of the membrane potential. Additionally, the proposition holds for both binary and integer-valued spike firing, showing that dynamic decay is a general and powerful alternative to reset. 


\textbf{From Serial Inference to Parallel Training.} After removing the reset mechanism, DSN retains the iterative nature of vanilla spiking neurons while supporting a parallelizable computation graph. This satisfies the three conditions described in Sec. \ref{sec: condition}, enabling parallel training. Specifically, Eq. \ref{eq:dsn_charge} can be rewritten into a general form determined solely by $\mathbf{X}_{1}, \mathbf{X}_{2}, ..., \mathbf{X}_{t}$:
\begin{align}
    &\mathbf{H}_t=\sum_{i=1}^{t}\left(\prod_{j=i+1}^t{\bm{\alpha}_j}\right)(\bm{1}-\bm{\alpha}_i)\odot \mathbf{X}_i.
\end{align}
Stacking $\mathbf{H}_{1}, \mathbf{H}_{2}, ..., \mathbf{H}_{T}$ gives $\mathbf{H} \in \mathbb{R}^{C\times T}$, and similarly for $\mathbf{X} \in \mathbb{R}^{C\times T}$. This allows expressing the iterative computation in a matrix form $\mathbf{H}=\mathbf{XW}$, where
\begin{align}
\mathbf{W}_{ij} =
    \begin{cases}
    \left( \prod_{k=i+1}^{j} \bm{\alpha}_k \right) (\bm{1}-\bm{\alpha}_i), & j \ge i\ \\
    0, & j < i
    \end{cases}.
\end{align}

\noindent Define $\mathbf{P},\mathbf{A}\in\mathbb{R}^{C\times T}$, $\mathbf{M}\in\mathbb{R}^{T\times T}$, where
\begin{align}
\mathbf{P}_j=\prod_{k=1}^j{\bm{\alpha}_k}, \mathbf{A}_i=\bm{\alpha}_i, \mathbf{M}_{ij}=\begin{cases}
    1, &j\geq i\\
    0, &j<i
    \end{cases}.
\end{align}
Then, the parallel form can be written as:
\begin{align}
    &\mathbf{H}=\mathbf{X}\left(\left(\left(\frac{\bm{1}-\mathbf{A}}{\mathbf{P}}\right)^T\mathbf{P}\right)\odot \mathbf{M}\right)
\end{align}
Here, $\frac{\mathbf{1}-\mathbf{A}}{\mathbf{P}}$ and $\odot$ denote element-wise division and product, respectively. During training, the dynamic decay $\mathbf{A}$, membrane potential $\mathbf{H}$ and their gradients can be computed rapidly in parallel\footnote{In practice, we avoid computing $\mathbf{H}$ via matrix multiplication due to numerical instability of $\mathbf{P}$ as the denominator \cite{yang2024gated}. Instead, we use a two-stage parallel scan algorithm \cite{martin2018parallelizing} for Eq. \ref{eq:dsn_charge} to derive $\mathbf{H}$.} with Triton-based acceleration operators \cite{tillet2019triton,yang2024fla}. During inference, we switch to Eq. \ref{eq:dsn_charge} for efficient serial inference, which requires to store only minimal states from the causal convolution and recurrent structure, thereby reducing both computational and memory overhead.

\section{Experiments}\label{experiments}
In this section, we evaluate the proposed DSN through a series of experiments in terms of 

\begin{itemize}
\item \textbf{Training Efficiency and Extrapolation.} We demonstrate that DSN achieves significant speedups over existing parallelizable spiking neurons on long sequences, and show through a text extrapolation experiment that serial inference is an inherent and indispensable property of spiking neurons.
\item \textbf{Generality.} The neuronal generality includes: i) The effectiveness of spiking neuron design, including the causal convolution structure and two spike activation modes (binary and integer); ii) Flexible adaptation to various network architectures, such as convolutional neural networks and Transformers; iii) Competitive performance across five different tasks. 
\item \textbf{Energy Consumption.} We discuss the potential of DSN for deployment on neuromorphic hardware and conduct a preliminary analysis of its energy efficiency compared with prior parallel spiking neurons.
\end{itemize}


\subsection{Training Efficiency and Extrapolation}

\begin{figure*}[t]
\vskip 0.2in
\begin{center}
\includegraphics[width=0.49\columnwidth]{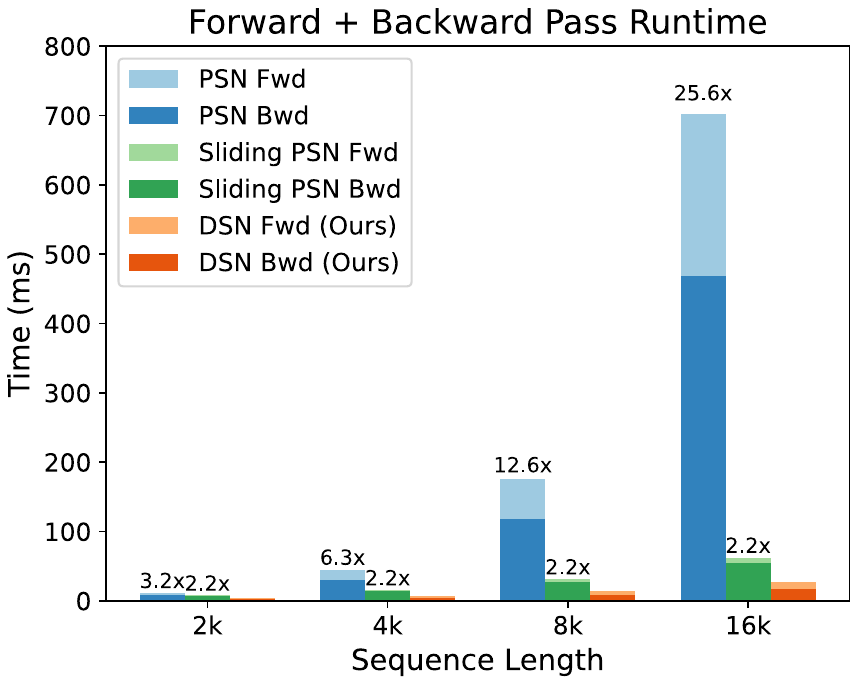}
\hfill
\includegraphics[width=0.48\columnwidth]{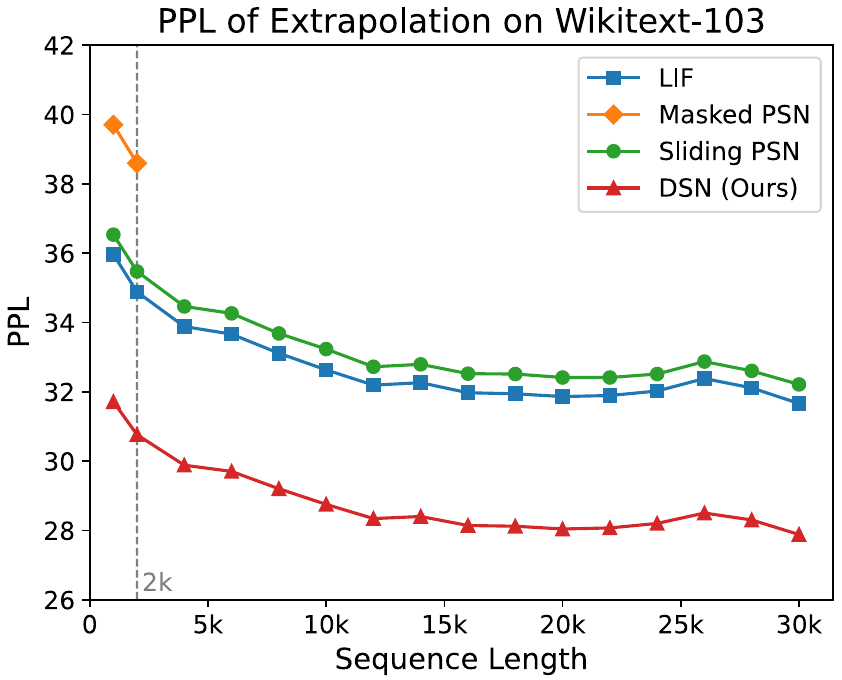}
\vspace{-5pt}
\caption{Results of training efficiency (left) and extrapolation (right). Left: Runtime of the forward and backward pass when training on sequences with lengths from 1k to 8k. Right: Perplexity (PPL, the lower the better) of extrapolation on Wikitext-103 with sequence lengths ranging from 1k to 30k, trained on sequences of length 2k.}\label{fig: speedup}
\end{center}
\vskip -0.2in
\end{figure*}

\textbf{Parallel Training Improves Efficiency.} We evaluate the training efficiency of different parallelizable spiking neurons by measuring the average runtime of 100 forward and backward passes under varying sequence lengths. From the results in Fig. \ref{fig: speedup} left, it can be observed that the speedup of DSN over PSN \citep{fang2023parallel} becomes increasingly pronounced as the sequence length grows. For sequences of length 16k, DSN achieves a 21.7× speedup in the forward pass and a 28.0× speedup in the backward pass compared to PSN, resulting in an overall runtime speedup of 25.6×. For Sliding PSN \citep{fang2023parallel}, a parameter-efficient variant of PSN, on sequences of length 16k, although DSN is slightly slower in the forward pass, its 3.2× faster backward pass leads to an overall training speed that is still 2.2× faster than Sliding PSN.

\textbf{Serial Inference Ensures Extrapolation.} In sequence tasks such as language modeling, models are often required to extrapolate to sequences longer than those seen during training. While vanilla spiking neurons can operate on sequences of arbitrary length, some parallelized spiking neurons lose serial inference ability and are therefore limited in extrapolation. To evaluate the extrapolation ability of DSN, we conduct experiments on WikiText-103 \citep{merity2017pointer} dataset using SpikingSSM \citep{shen2025spikingssms} architecture with different spiking neurons. Models are trained on sequences of length 2k and evaluated on longer sequences by measuring perplexity. As shown in Fig. \ref{fig: speedup} right, masked PSN \citep{fang2023parallel} fails beyond 2k tokens, whereas DSN maintains stable performance on sequences up to 30k tokens and outperforms LIF and sliding PSN \citep{fang2023parallel}.

\subsection{Generality}

\subsubsection{Effectiveness of Spiking Neuron Design}
We evaluate the effectiveness of spiking neuron design on Sequential CIFAR10, including the causal convolution structure and different spike activation modes. Results in Table \ref{tab:ablation_dynamic_decay} and \ref{tab:ablation_spike_firing} show that: 

\begin{wraptable}{r}{0.45\textwidth}
\vspace{-10pt}
\small
\renewcommand{\arraystretch}{1.1}
\setlength{\tabcolsep}{3pt}

\centering
\caption{Ablations: different design of dynamic decay structure.}
\label{tab:ablation_dynamic_decay}
\vspace{6pt}
\begin{tabular}{cc}
\toprule
Methods & Accuracy (\%) \\ \midrule
\multicolumn{1}{l}{causal conv} & 90.10 \\
\multicolumn{1}{l}{fully connected layer} & 89.28 \\
\multicolumn{1}{l}{low-rank mapping} & 86.72 \\
\multicolumn{1}{l}{inter-channel conv} & 86.76 \\
\multicolumn{1}{l}{w/o conv} & 84.53 \\
\bottomrule
\end{tabular}

\vspace{4pt}
\setlength{\tabcolsep}{4pt}

\caption{Ablations: spike activation modes. $N$: the maximum spike count. Rect: Rectangular. ATan: Arc Tangent.}
\label{tab:ablation_spike_firing}
\vspace{6pt}
\begin{tabular}{ccc}
\toprule
$N$ & Surrogate Grad & Accuracy (\%) \\ \midrule
4 & Rect & 90.10 \\
3 & Rect & 89.56 \\
2 & Rect & 89.28 \\ \hline
1 & Rect & 86.74 \\
1 & ATan & 87.45 \\
\bottomrule
\end{tabular}

\vspace{-10pt}
\end{wraptable}

First, dynamic decay is the primary contributor to performance. Removing the causal convolution eliminates meaningful neuronal dynamics and leads to a significant degradation (84.53\%). We also test alternative decay structures, including fully connected layers for inter-channel interaction, low-rank mappings, and inter-channel convolution. In comparison, causal convolution remains a simple yet effective design.

Second, dynamic decay is effective for both binary ($N=1$) and integer-valued ($N>1$) spike firing. In the binary mode, using Arc Tangent as the surrogate gradient function yields better performance. In the integer-valued mode, the performance improves with increasing spike limit $N$.

Based on the above experiments, we design DSN by combining causal convolution with an integer-valued training technique, and validate its generality across different model architectures and tasks.

\subsubsection{Generality Across Multiple Tasks}

We evaluate DSN across five tasks: image classification, neuromorphic event processing, time-series forecasting, reinforcement learning, and language modeling, using models that span convolutional neural networks, Transformers, and State Space Models (SSMs). Competitive results across multiple datasets and network architectures demonstrate the general effectiveness of DSN.

\begin{table}[t]
\small
\caption{Performance on Sequential CIFAR. The timestep is 32. Param: parameters (M).}
\label{tab: Sequential CIFAR results}
\vspace{4pt}
\centering
\renewcommand{\arraystretch}{1.1}
\setlength{\tabcolsep}{3pt}
\begin{tabular}{ccccccc}
\toprule
\multirow{2}{*}[-0.5ex]{Methods} & \multirow{2}{*}[-0.5ex]{Parallel} & \multirow{2}{*}[-0.5ex]{Serial}&  \multicolumn{2}{c}{S-CIFAR10} & \multicolumn{2}{c}{S-CIFAR100} \\ \cmidrule{4-7}
 & & & Param. & Accuracy (\%)& Param. & Accuracy (\%) \\ \midrule
 LIF \citep{abbott1999lapicque} & \xmark& \cmark & $0.513$ & $81.50$ & $0.537$ & $55.45$ \\
 PLIF \citep{fang2021incorporating} & \xmark& \cmark & $0.513$ & $81.50$ & $0.537$ & $55.45$ \\
 PSN \citep{fang2023parallel}  & \cmark& \xmark & $0.521$ &$88.45$ & $0.544$& $62.21$   \\
 RPSU \citep{li2024parallel}  & \cmark& \xmark & $0.517$ &$84.44$ & $0.540$ &$57.89$ \\
 IPSU \citep{li2024parallel}  & \cmark& \xmark & $0.517$ &$87.28$ & $0.540$ &$59.76$ \\
 Sliding PSN \citep{fang2023parallel} & \cmark& \cmark & $0.514$ &$86.70$ & $0.538$ &$62.11$ \\ \midrule
 \bf{DSN (Ours)}  & \cmark& \cmark & $0.519$ &$\bf90.10$ &  $0.542$ &$\bf64.70$   \\ 
 \bottomrule
\end{tabular}
\end{table}

\begin{table}[!t]
\small
\vspace{-4pt}
\caption{Performance on ImageNet and CIFAR10-DVS. T: Timesteps. }
\label{tab: imagenet and dvs results}
\vspace{4pt}
\centering
\renewcommand{\arraystretch}{1.1}
\setlength{\tabcolsep}{3pt}
\resizebox{\linewidth}{!}{
\begin{tabular}{ccccccc}
\toprule
 Dataset &   Methods &  Architecture &  Parallel & Serial & T & Accuracy (\%)\\ \midrule
 \multirow{6}{*}[0.0ex]{ImageNet} 
 & MBPN \citep{guo2023mbpn}  & ResNet18  & \xmark& \cmark &$4$ & $63.14$   \\
 & InfLoR-SNN \citep{guo2022lnflor-snn} & ResNet18 &  \xmark& \cmark &$4$ & $64.78$ \\
 & SEW ResNet \citep{fang2021deep}  & SEW ResNet18 &  \xmark& \cmark &$4$ & $63.18$   \\
 & PMSN \citep{chen2024pmsn}  & SEW ResNet18  & \cmark& \cmark&$4$ & $66.64$   \\
 & PSN \citep{fang2023parallel}  & SEW ResNet18  & \cmark& \xmark&$4$ & $67.63$   \\ \cmidrule{2-7}
 & \bf{DSN (Ours)}  & SEW ResNet18 &  \cmark& \cmark&$4$ & $\bf{68.21}$   \\ \midrule
 \multirow{7}{*}[0.0ex]{CIFAR10-DVS} & SEW ResNet \citep{fang2021deep}  & Wide 7B Net  & \xmark& \cmark&$16$ & $74.40$   \\
 & GLIF \citep{yao2022glif}  & Wide 7B Net  & \xmark& \cmark&$16$ & $78.10$   \\
 & DeepTAGE \citep{liu2025deeptage}  & VGG-11  & \xmark& \cmark&$10$ & $81.23$   \\
 & RPSU \citep{li2024parallel}  & VGGSNN  & \cmark& \xmark&$10$ & $82.00$   \\
 & FPT \citep{feng2025fpt}  & VGG-11  & \cmark& \xmark&$10$ & $\bf{85.50}$   \\
 & sliding PSN \citep{fang2023parallel}  & VGGSNN   & \cmark& \cmark&$4,8$ & $82.30, 85.30$   \\ \cmidrule{2-7}
 & \bf{DSN (Ours)}  & VGGSNN  & \cmark& \cmark &$4,8$ & $83.90, 85.30$   \\ 
\bottomrule
\end{tabular}
}
\end{table}

\begin{table*}[t]
\centering
\renewcommand{\arraystretch}{1.1}
\vspace{-4pt}
\caption{Experimental results of 3 time-series forecasting tasks with prediction lengths $L=6,24,48,96$.  $\uparrow$ ($\downarrow$) indicates the higher (lower) the better. All results are averaged across $3$ random seeds. The leading zero before the decimal point is omitted. Param: parameters (M).}
\vspace{4pt}
\label{tab:time-series tasks}
\setlength{\tabcolsep}{3pt}
\resizebox{\linewidth}{!}{
\begin{tabular}{cccc|cccc|cccc|cccc|c}\toprule
\multirow{2}{*}[-0.5ex]{\bf Methods} & \multirow{2}{*}[-0.5ex]{\bf Spike} & \multirow{2}{*}[-0.5ex]{\bf Param.}& \multirow{2}{*}[-0.5ex]{\bf Metric} & \multicolumn{4}{c|}{\bf Metr-la} & \multicolumn{4}{c|}{\bf Pems-bay} & \multicolumn{4}{c|}{\bf Solar} & \multirow{2}{*}[-0.5ex]{\textbf{Avg.}} \\
\cmidrule{5-16}
& & & & $6$ & $24$ & $48$ & $96$ & $6$ & $24$ & $48$ & $96$& $6$ & $24$ & $48$ & $96$ \\
\midrule
\multirow{2}{*}{Transformer} & \multirow{2}{*}{\xmark}& \multirow{2}{*}{$2.53$}& R$^2$$\uparrow$ & $.727$ & $\bf{.554}$ & $.413$ & $\bf{.284}$ & $.785$ & $.734$ & $.688$ & $.673$ & $.953$ & $.858$ & $.759$ & $.718$ & $.679$ \\
& & & RSE$\downarrow$ & $.551 $ & $\bf{.704} $ & $.808 $ & $\bf{.895}  $ & $.502 $ & $.558 $ & $.610 $ & $.618 $ & $.223 $ & $.377 $ & $.504 $ & $\bf.545$  & $.575$ \\ 
\midrule
\multirow{2}{*}{Spikformer} & \multirow{2}{*}{\cmark}& \multirow{2}{*}{$2.52$}& R$^2$$\uparrow$ & $.713$ & $.527$ & $.399$ & $.267$ & $.773$ & $.697$ & $.686$ & $.667$ & $.929$& $.828$ & $.744$ & $.674$ & $.659$ \\
& & & RSE$\downarrow$ & $.565$ & $.725$ & $.818$ & $.903$ & $.514$ & $.594$ & $.606$ & $.621$ & $.272$ & $.426$ & $.519$ & $.586$  & $.596$ \\ 
\midrule
\multirow{2}{*}{\makecell[c]{Spikformer \\ w/ PSN}} & \multirow{2}{*}{\cmark} & \multirow{2}{*}{$2.68$}& R$^2$$\uparrow$ & $.716$ & $.518$ & $.401$ & $.268$ & $.738$ & $.671$ & $.666$ & $.639$ & $.861$& $.759$ & $.554$ & $.439$  & $.603$ \\
& & & RSE$\downarrow$ & $.562$ & $.731$ & $.815$ & $.901$ & $.553$ & $.620$ & $.624$ & $.649$ & $.383$ & $.504$ & $.685$ & $.749$  & $.648$ \\ 
\belowcolormidrule
\rowcolor{gray!20}
 & & & R$^2$$\uparrow$ & $\bf{.734}$ & $.549$ & $\bf{.422}$ & $.283$ & $\bf{.807}$ & $\bf.745$ & $\bf{.696}$ & $\bf.683$ & $\bf.956$ & $\bf.860$ & $\bf.765$ & $\bf.736$  & $\bf{.686}$ \\
\rowcolor{gray!20}
\multirow{-2}{*}{\makecell[c]{\textbf{Spikformer} \\ \textbf{w/ DSN}}} & \multirow{-2}{*}{\cmark} &\multirow{-2}{*}{$2.68$}& RSE$\downarrow$ & $\bf{.539}$ & $.720$ & $\bf{.804}$ & $.896$ & $\bf{.475}$ & $\bf.538$ & $\bf.581$ & $\bf.594$ & $\bf.219$& $\bf.373$ & $\bf.481$ & $.572$  & $\bf{.566}$ \\ 
\abovecolormidrule
\midrule
\multirow{2}{*}{iTransformer} & \multirow{2}{*}{\xmark}&\multirow{2}{*}{$1.63$} & R$^2$$\uparrow$ & $\bf{.829} $ & $.623 $ & $.439 $ & $\bf.285$ & $\bf{.887}$ & $.719 $ & $.685 $ & $.668 $ & $\bf{.964} $ & $\bf{.879} $ & $\bf{.799} $ & $\bf{.738} $ & $.710$ \\
& & & RSE$\downarrow$ & $\bf{.436} $ & $.648 $ & $.780 $ & $\bf.878 $ & $\bf{.362} $ & $.547 $ & $\bf{.561} $ & $.584$ & $\bf{.191} $ & $.348 $ & $\bf{.448} $ & $.563 $  & $.529$ \\ 
\midrule
\multirow{2}{*}{iSpikformer} & \multirow{2}{*}{\cmark}&\multirow{2}{*}{$1.63$} & R$^2$$\uparrow$ & $.817$ & $.618$ & $\bf.440 $ & $.279$ & $.879$ & $\bf.744 $ & $.687 $ & $\bf.674$ & $.961$ & $.876$ & $.795$ & $\bf{.738}$  & $.709$ \\
& & & RSE$\downarrow$ &$.475 $ &$.668$ &$\bf{.752}$ & $.905$ & $.376$ & $\bf{.536}$ & $.569$ & $\bf{.580}$ & $.204$ & $\bf{.333}$ & $.465$ & $\bf{.521}$  & $.532$ \\ 
\belowcolormidrule
\rowcolor{gray!20}
 & & & R$^2$$\uparrow$ & $.823$ & $\bf{.624}$ & $\bf.440$ & $.283$ & $.883$ & $.740$ & $\bf.689$ & $.672$ & $\bf{.964}$ & $\bf{.879}$ & $.798$ & $.736$  & $\bf{.711}$ \\
\rowcolor{gray!20}
\multirow{-2}{*}{\makecell[c]{\textbf{iSpikformer} \\ \textbf{w/ DSN}}} & \multirow{-2}{*}{\cmark}&\multirow{-2}{*}{$1.79$} & RSE$\downarrow$ & $.450$ & $\bf.646$ & $.755$ & $.881$ & $.368$ & $.541$ & $.564$ & $.583$ & $.199$ & $.350$ & $.450$ & $.526$  & $\bf{.526}$ \\ 
\colorbottomrule
\end{tabular}
}
\end{table*}

\textbf{Sequential CIFAR.} we build a convolution-based SNN on Sequential CIFAR dataset, where the images from CIFAR \citep{krizhevsky2009learning} are input into the model in sequential pixel form and the timestep is equal to the width 32 of the images. The experimental setup and other hyperparameters are kept consistent with those of PSN \citep{fang2023parallel}. Results in Table \ref{tab: Sequential CIFAR results} show that our DSN achieves state-of-the-art performance under the same parameter scale.

\textbf{ImageNet.} We further evaluate the performance of DSN on this larger-scale image classification task \citep{deng2009imagenet}. The experimental settings are identical to \cite{fang2023parallel}. For a fair comparison, we follow \citep{yao2025scaling} and set the training step $T$ of DSN to 1, while treating the upper bound $N=4$ of integer-spike emission as the extended timestep $T$. As illustrated in Table \ref{tab: imagenet and dvs results}, our method still achieves relatively higher accuracy among parallel spiking neurons.

\textbf{CIFAR10-DVS.} To validate the effectiveness of our method in processing neuromorphic events, we select CIFAR10-DVS \citep{li2017cifar10dvs} as the benchmark and adopt the VGG architecture from \cite{deng2022vgg}. As shown in Table \ref{tab: imagenet and dvs results}, DSN shows performance comparable to sliding PSN and ranks just below FPT \citep{feng2025fpt}, which employs a PSN-like structure and therefore cannot perform serial inference.

\textbf{Time-series Forecasting Tasks.} On more realistic time-series forecasting tasks, we adapt DSN to the following datasets: Metr-la \citep{li2018diffusion}: traffic flow records from Los Angeles; Pems-bay \citep{li2018diffusion}: traffic flow records from the San Francisco Bay Area; Solar \citep{lai2018modeling}: solar power generation data. Baseline architectures include Transformer \citep{vaswani2017attention}, iTransformer \citep{liu2024itransformer}, and their respective SNN counterparts \citep{zhou2023spikformer, lv2024efficient}. For all SNN-based time-series forecasting models, we replace the original LIF neurons with DSN and make architectural modifications (see \ref{detail_timeseries_forecasting}). The Root Relative Squared Error (RSE) and the coefficient of determination ($\text{R}^2$) are used as metrics. It can be seen from Table \ref{tab:time-series tasks} that DSN-based architectures exhibit superior performance on various tasks and prediction lengths.

\begin{table}[t]
\small
\vspace{-6pt}
\caption{Reinforcement learning results on three representative benchmarks.}
\label{tab: reinforcement learning}
\vspace{6pt}
\centering
\renewcommand{\arraystretch}{1.1}
\begin{tabular}{ccccc}
\toprule
Methods & IDP-v4 & Hopper-v4 & Walker2d-v4 \\
\midrule
ANN (TD3) & 7503 $\pm$ 3713 & 3410 $\pm$ 164 & 4340 $\pm$ 383 \\
ANN-SNN \citep{bu2025inference}& 3859 $\pm$ 4440 & 3098 $\pm$ 281 & 4235 $\pm$ 354 \\
Vanilla LIF & 9347 $\pm$ 1 & 3520 $\pm$ 94 & 1862 $\pm$ 1450 \\
pop-SAN \citep{tang2021deep} & 9351 $\pm$ 1 & 2772 $\pm$ 1263 & 3307 $\pm$ 1514  \\
MDC-SAN \citep{zhang2022multi}& 9350 $\pm$ 1 & 3446 $\pm$ 131 & 3964 $\pm$ 1353 \\
ILC-SAN \citep{chen2024fully}& 9352 $\pm$ 1 & 3403 $\pm$ 148 & 4200 $\pm$ 717 \\
PT-LIF \citep{xu2025proxy} & 9348 $\pm$ 1 & 3385 $\pm$ 157 & 4314 $\pm$ 423 \\ \midrule
\bf DSN (Ours) & \textbf{9354} $\pm$ 1& $ \textbf{3565} \pm 68 $ & \textbf{4436} $\pm$ 196 \\
\bottomrule
\end{tabular}
\vspace{-8pt}
\end{table}

\textbf{Reinforcement Learning.}
Following \citet{xu2025proxy}, we evaluate DSN on three off-policy reinforcement learning benchmarks: IDP-v4~\citep{todorov2014convex}, Hopper-v4~\citep{erez2012infinite}, and Walker2d-v4.
Table~\ref{tab: reinforcement learning} reports the average returns of DSN across these environments, in comparison with several representative baselines, including ANN-based RL methods, ANN-to-SNN conversion approaches~\citep{bu2025inference}, and state-of-the-art SNN-based RL algorithms such as pop-SAN~\citep{tang2021deep}, MDC-SAN~\citep{zhang2022multi}, ILC-SAN~\citep{chen2024fully}, and PT-LIF~\citep{xu2025proxy}. Detailed experimental configurations are provided in \ref{app:rl}.
Overall, DSN achieves state-of-the-art performance on all three benchmarks. Moreover, compared with existing SNN-based methods, DSN consistently attains both superior average returns and reduced variance, indicating improved stability during training.

\begin{wraptable}{r}{0.48\textwidth}
\small
\vspace{-19pt}
\caption{Experimental results on WikiText-103 dataset. $\downarrow$ indicates the lower the better. $^\dag$ means the results reported by \citep{zhong2024spike}. Param: parameters (M). PPL: perplexity.}
\label{tab: wikitext-103 task}
\vspace{8pt}
\centering
\renewcommand{\arraystretch}{1.1}
\setlength{\tabcolsep}{3pt}
\begin{tabular}{ccc}
\toprule
 Methods &   Param. &  PPL $\downarrow$  \\ \midrule
 SpikeGPT$^\dag$ \citep{zhu2024spikegpt} &  $213$ & $39.75$    \\
 SPikE-SSM$^\dag$ \citep{zhong2024spike} &  $75.4$ & $33.18$    \\
 SpikingSSM \citep{shen2025spikingssms} & $75.4$ & $32.25$    \\
 \bf DSN (Ours) &  $75.5$ & $\bf 28.50$    \\
\bottomrule
\end{tabular}
\vspace{-10pt}
\end{wraptable}

\textbf{WikiText-103.} To demonstrate that DSN can model more complex sequences such as language, we evaluate its perplexity on WikiText-103 \citep{merity2017pointer}, a large-scale word-level dataset constructed from the English Wikipedia. Following SpikingSSM \citep{shen2025spikingssms}, we use the S4D architecture \citep{gu2022efficiently} and replace its spiking neurons with DSN. Results in Table \ref{tab: wikitext-103 task} show that DSN performs the best among spiking language models.

\subsection{Energy Consumption}

\textbf{Neuromorphic Chip Deployment.} There are still some gaps before DSN can be deployed on neuromorphic chips. On the one hand, dynamic decay introduces more complex floating-point operations, such as convolutions and Sigmoid functions. To alleviate the hardware adaptation burden, further algorithmic optimizations are possible. For example, instead of performing multiplications between floating-point weights and activations, the weights can be quantized to ternary values $(-1, 0, 1)$ \citep{ma2024era, zhu2024scalable}, or a spiking function can be inserted to convert floating-point inputs into spikes in advance. For the Sigmoid function, the exponential base can be replaced with 2 to enable lightweight computations using shift operations \citep{tang2025sorbet}. On the other hand, prior work shows that integer outputs can be converted into asynchronously emitted spikes on neuromorphic hardware \citep{yao2025scaling}, providing a clear direction for hardware-oriented adaptation of optimized spike firing patterns.

\textbf{Energy Consumption.} Compared to vanilla spiking neurons, DSN affects energy consumption in two main aspects. i) The adaptive regulation of membrane potential help reduce the spike firing rate, thereby lowering the overall energy cost. ii) The complex floating-point operations in dynamic decay structures inevitably incur additional energy overhead. To evaluate the potential of DSN in terms of energy efficiency, we estimate the energy cost of the Sequential CIFAR network with different spiking neurons by following the method in \citep{yao2023sdsa}. For internal neuron operations, we refer to \citep{tang2025sorbet} where exponential operations in functions such as Softmax or Sigmoid can be transformed into additions and bit-shift operations. Results in Table \ref{tab: energy consumption} show that although additional modules such as convolution operations were introduced, the total energy consumption of DSN can be slightly lower than that of LIF due to its reduced spike firing rate. Additionally, PSN exhibits the highest spike firing rate, which in turn results in relatively high energy consumption. See \ref{detail_energy_consumption} for more details.

\begin{table}[t]
\small
\caption{Spike firing rate (SFR.) and energy cost (mJ) of different methods. }
\label{tab: energy consumption}
\vspace{4pt}
\centering
\renewcommand{\arraystretch}{1.1}
\begin{tabular}{ccccc}
\toprule
\multirow{2}{*}[-0.5ex]{Methods} & \multicolumn{2}{c}{S-CIFAR10} & \multicolumn{2}{c}{S-CIFAR100} \\ \cmidrule{2-5}
 &  SFR. & Energy Cost& SFR. & Energy Cost \\ \midrule
 LIF \citep{abbott1999lapicque} & $0.1499$ & $107.80$ & $0.1697$ & $121.78$ \\
 PSN \citep{fang2023parallel}  &  $0.2143$ &$235.87$ & $0.2226$& $242.03$   \\
 Sliding PSN \citep{fang2023parallel} &  $0.1820$ &$170.39$ & $0.1900$ &$176.22$ \\ 
 \bf{DSN (Ours)}  & $\bf0.1238$ &$\bf102.89$ &  $\bf0.1324$ &$\bf108.94$   \\ 
 \bottomrule
\end{tabular}
\end{table}

\section{Conclusion}
In this paper, we identify a critical limitation in existing efforts toward parallel training in SNNs: the neglect of preserving essential characteristics of vanilla spiking neurons, including the functions of the reset mechanism and the capability for serial inference. Under a new functional viewpoint, we summarize the functions of the reset mechanism in vanilla spiking neurons as introducing nonlinearity and controlling membrane potential. Meanwhile, we identify three conditions that spiking neurons need to meet in order to enable both parallel training and serial inference. Based on this general perspective, we introduce a dynamic decay spiking neuron that offers improved functions compared to reset while remaining compatible with serial inference. We verify the competitive training efficiency, stable extrapolation, generality across multiple tasks, and energy consumption of our method. Our work offers new insights into the exploration of high-performance spiking neurons with efficient training and inference abilities in the era of foundation models.

\section*{Reproducibility Statement}
The authors have made great efforts to ensure the reproducibility of the empirical results reported in this paper.
The experiment settings, evaluation metrics, and datasets were described in detail in \ref{sec: detail_experiment}.
Additionally, we had submitted the source code of the proposed training algorithm with our paper, and plan to release the source code on GitHub upon acceptance. 

\section*{Acknowledgments}
The authors would like to thank the anonymous reviewers for their valuable comments.



\bibliographystyle{unsrtnat} 
\bibliography{main}



\newpage
\appendix
\section{Details of Theoretical Analysis}
\subsection{Control Ability over Membrane Potential of Reset}
In this section, building on the definitions of $\Delta$-short control and long control presented in the main text, we conduct a more rigorous analysis of how the reset mechanism in vanilla spiking neurons controls the membrane potential, which helps us understand its limitations.

Similar to LIF neurons, the charging equation for IF neurons can be written as:
\begin{align}
    &H_{t}=H_{t-1}+{X}_{t}\label{eq:IF_charge}
\end{align}
We combine the Eq. \ref{eq:LIF_reset} and Eq. \ref{eq:IF_charge}, and derive one-step iteration of the membrane potential:
\begin{align}
    &H_{t}=
    \begin{cases}
    (1-f(H_{t-1}))H_{t-1}+X_t, &\text{hard reset}\\
    H_{t-1}-V_\text{th}f(H_{t-1})+X_t , &\text{soft reset}
    \end{cases}.
\end{align}
Where $f$ is the firing function. $f(H_t)=1$ when $H_t\geq V_\text{th}$; otherwise $f(H_t)=0$. With these prerequisites established, we discuss whether IF and LIF neurons strictly possess $\Delta$-short and long control abilities over the membrane potential under hard reset or soft reset conditions.

\begin{proposition}
Hard reset in IF neurons has both $\Delta$-short and long control abilities over the membrane potential.
\end{proposition}
\vspace{-8pt}
\begin{proof}
Firstly, suppose $H_{t-\Delta}\geq V_{\text{th}}$ and $X_{t-\Delta+1}$, ..., $X_t < V_{\text{th}}/\Delta$, then 
\begin{equation}
    \begin{aligned}
        H_{t-\Delta+1}&=0+{X}_{t-\Delta+1}<V_{\text{th}}/\Delta\\
        H_{t-\Delta+2}&=H_{t-\Delta+1}+{X}_{t-\Delta+2}<V_{\text{th}}/\Delta+V_{\text{th}}/\Delta=2V_{\text{th}}/\Delta\\
        H_{t-\Delta+3}&=H_{t-\Delta+2}+{X}_{t-\Delta+3}<2V_{\text{th}}/\Delta+V_{\text{th}}/\Delta=3V_{\text{th}}/\Delta\\
        ...\\
        H_{t-\Delta+\Delta}&=H_{t-\Delta+\Delta-1}+{X}_{t-\Delta+\Delta}<(\Delta-1)V_{\text{th}}/\Delta + V_{\text{th}}/\Delta=V_{\text{th}}\\
    \end{aligned}
\end{equation}
Thus, we obtain $H_t<V_{\text{th}}$.

Secondly, suppose $X_i\leq C, i=1,...,t$. It is easy to get that $H_1=X_1< C+V_{\text{th}}$. Besides, if $H_{t-1} < C+V_{\text{th}}$, then
\begin{equation}
    \begin{aligned}
        H_t&=(1-f(H_{t-1}))H_{t-1}+X_t\\
        &=
        \begin{cases}
        H_{t-1}+X_t<C+V_{\text{th}}, &H_{t-1}<V_{\text{th}}\\
        X_t < C+V_{\text{th}}, &V_{\text{th}}\leq H_{t-1}< C+V_{\text{th}}
        \end{cases}.
    \end{aligned}
\end{equation}
By mathematical induction, we know that $\{H_t\}$ has an upper bound $C+V_{\text{th}}$. 

From Definition \ref{local_control} and \ref{global_control}, IF with hard reset has both $\Delta$-short and long control abilities.
\end{proof}

\begin{proposition}
Hard reset in LIF neurons has both $\Delta$-short and long control abilities over the membrane potential.
\end{proposition}
\vspace{-8pt}
\begin{proof}
Firstly, suppose $H_{t-\Delta}\geq V_{\text{th}}$ and $X_{t-\Delta+1}$, ..., $X_t < V_{\text{th}}/\Delta$, then 
\begin{equation}
    \begin{aligned}
        H_{t-\Delta+1}&=0+(1-\beta){X}_{t-\Delta+1}<V_{\text{th}}/\Delta\\
        H_{t-\Delta+2}&=\beta H_{t-\Delta+1}+(1-\beta){X}_{t-\Delta+2}<V_{\text{th}}/\Delta+V_{\text{th}}/\Delta=2V_{\text{th}}/\Delta\\
        H_{t-\Delta+3}&=\beta H_{t-\Delta+2}+(1-\beta){X}_{t-\Delta+3}<2V_{\text{th}}/\Delta+V_{\text{th}}/\Delta=3V_{\text{th}}/\Delta\\
        ...\\
        H_{t-\Delta+\Delta}&=\beta H_{t-\Delta+\Delta-1}+(1-\beta){X}_{t-\Delta+\Delta}<(\Delta-1)V_{\text{th}}/\Delta + V_{\text{th}}/\Delta=V_{\text{th}}\\
    \end{aligned}
\end{equation}
Thus, we obtain $H_t<V_{\text{th}}$.

Secondly, suppose $X_i\leq C, i=1,...,t$. It is easy to get that $H_1=(1-\beta)X_1\leq C$. Besides, if $H_{t-1} \leq C$, then
\begin{align}
    &H_{t}=\beta (1-f(H_{t-1}))H_{t-1}+(1-\beta){X}_{t}\leq\beta C+(1-\beta)C=C
\end{align}
By mathematical induction, we know that $\{H_t\}$ has an upper bound $C$. 

From Definition \ref{local_control} and \ref{global_control}, LIF with hard reset has both $\Delta$-short and long control abilities.
\end{proof}

\noindent\textit{Remark}: Hard reset immediately clears the membrane potential at the current timestep. However, regardless of input magnitude, its effect lasts only one timestep and cannot adaptively adjust its range. Therefore, although hard reset possesses the ability of $\Delta$-short control, its actual control window collapses to a fixed value of $\Delta=1$.

We next turn to soft reset, which controls the membrane potential by subtracting a fixed value. We will show that, due to its lack of flexibility, soft reset is also not an effective mechanism for controlling the membrane potential. Before that, we introduce a lemma needed for the formal proof.

\begin{lemma}\label{lemma_for_soft_reset}
For any positive integer $\Delta$ and any positive integer $m \le \Delta$, we have
\begin{align}
    \Delta+\frac{m}{\Delta}-m\ge 1
\end{align}
\end{lemma}
\begin{proof}
When $m=\Delta$, the equality holds.

When $1\le m<\Delta$, consider the function
\begin{align}
    f(x)=\Delta+\frac{x}{\Delta}-x=(\frac{1}{\Delta}-1)x+\Delta, \quad x \in [1, \Delta], \Delta>1.
\end{align}
Since $f(x)$ is monotonically decreasing in $x$, $f(m)>f(\Delta)=\Delta+1-\Delta=1$.
\end{proof}

\begin{proposition}
Soft reset in IF neurons has neither $\Delta$-short nor long control abilities over the membrane potential.
\end{proposition}
\vspace{-8pt}
\begin{proof}
Firstly, suppose $H_{t-\Delta}\geq V_{\text{th}}$ and $X_{t-\Delta+1}$, ..., $X_t < V_{\text{th}}/\Delta$. If $H_{t-\Delta}>(\Delta+1)V_{\text{th}}-\sum_{i=1}^\Delta X_{t-\Delta+i}>V_{\text{th}}$, then, in combination with Lemma \ref{lemma_for_soft_reset}, we have
\begin{equation}
    \begin{aligned}
        H_{t-\Delta+1}&=H_{t-\Delta}-V_{\text{th}}+{X}_{t-\Delta+1}\\
        &>\Delta \cdot V_{\text{th}}-\sum_{i=2}^\Delta X_{t-\Delta+i}>(\Delta+\frac{1}{\Delta}-1)V_{\text{th}}\ge V_{\text{th}}\\
        H_{t-\Delta+2}&=H_{t-\Delta+1}-V_{\text{th}}+{X}_{t-\Delta+2}\\
        &>(\Delta-1)V_{\text{th}}-\sum_{i=3}^\Delta X_{t-\Delta+i}>(\Delta+\frac{2}{\Delta}-2)V_{\text{th}}\ge V_{\text{th}}\\
        ...\\
        H_{t-\Delta+\Delta-1}&=H_{t-\Delta+\Delta-2}-V_{\text{th}}+{X}_{t-\Delta+\Delta-1}\\
        &>2V_{\text{th}}-X_{t}>(\Delta+\frac{\Delta-1}{\Delta}-(\Delta-1))V_{\text{th}}\ge V_{\text{th}}\\
        H_{t-\Delta+\Delta}&=H_{t-\Delta+\Delta-1}-V_{\text{th}}+{X}_{t-\Delta+\Delta}>V_{\text{th}}\\
    \end{aligned}
\end{equation}
Thus, we obtain $H_t>V_{\text{th}}$. From Definition \ref{local_control}, $\Delta$-short control ability does not hold in this case.

Secondly, let $X_i=C>V_{\text{th}},i=1,...,t$. For any $C_H>V_{\text{th}}$, there exists $t_0\geq\frac{C_H-V_{\text{th}}}{C-V_{\text{th}}}$ such that when $t>t_0$, 
\begin{equation}
    \begin{aligned}
        H_{t}&=H_{t-1}-V_{\text{th}}+{X}_{t}\\
        &=H_{t-2}+(X_{t-1}+X_t)-2V_{\text{th}}=...\\
        &=\sum_{i=1}^t X_i -(t-1)V_{\text{th}}\\
        &=tC-(t-1)V_{\text{th}}>C_H
    \end{aligned}
\end{equation}
In this case, $\{H_t\}$ does not have an upper bound. From Definition \ref{global_control}, long control ability does not hold.
\end{proof}
\begin{proposition}
Soft reset in LIF neurons does not have $\Delta$-short control, but has long control ability over the membrane potential.
\end{proposition}
\vspace{-8pt}
\begin{proof}
Firstly, suppose $H_{t-\Delta}\geq V_{\text{th}}$ and $X_{t-\Delta+1}$, ..., $X_t < V_{\text{th}}/\Delta$. To make $\Delta$-short control fail, it suffices to have $H_t = \beta (H_{t-1}-V_{\text{th}})+(1-\beta){X}_{t}\ge V_{\text{th}}$, which requires 
\begin{align}
    H_{t-1}\ge(1+\frac{1}{\beta})V_{\text{th}}+(1-\frac{1}{\beta})X_t
\end{align}
Since $H_{t-1} = \beta (H_{t-2}-V_{\text{th}})+(1-\beta){X}_{t-1}$, this implies
\begin{align}
    H_{t-2}\ge(1+\frac{1}{\beta}+\frac{1}{\beta^2})V_{\text{th}}+(1-\frac{1}{\beta})X_{t-1} + \frac{1}{\beta}(1-\frac{1}{\beta})X_{t}
\end{align}
By recursion, we need
\begin{align}
    H_{t-\Delta}\ge\sum_{i=0}^{\Delta}\frac{1}{\beta^i}V_{\text{th}}+\sum_{i=1}^{\Delta}\frac{1}{\beta^{i-1}}(1-\frac{1}{\beta})X_{t-\Delta+i}
\end{align}

From Definition \ref{local_control}, $\Delta$-short control ability does not hold in this case. 

Secondly, suppose $X_i\leq C, i=1,...,t$. It is easy to get that $H_1=(1-\beta)X_1\leq C$. Besides, if $H_{t-1} \leq C$, then
\begin{align}
    &H_{t}=\beta (H_{t-1}-V_{\text{th}}f(H_{t-1}))+(1-\beta){X}_{t}\leq\beta C+(1-\beta)C=C
\end{align}
By mathematical induction, we know that $\{H_t\}$ has an upper bound $C$. From Definition \ref{global_control}, long control ability holds. 
\end{proof}

\noindent\textit{Remark}: In soft reset, for any given upper bound $\Delta$ on the control window, there always exists a sufficiently large input that extends its effect beyond this bound, leading to continuous spike firing. Moreover, the subtracted value is fixed, while the accumulation of the membrane potential varies dynamically. When the accumulation rate exceeds the subtraction rate, the membrane potential may explode even with soft reset mechanism. This situation can only be prevented by introducing an additional leakage factor $\beta$.

\subsection{Proof of Functions of Dynamic Decay}\label{detail_dynamic_decay_proof}

\begin{proposition}
Dynamic decay can introduce nonlinearity and enabling more flexible $\Delta$-short and long control of the membrane potential than the reset mechanism.
\end{proposition}
\vspace{-8pt}
\begin{proof}
Firstly, with dynamic decay, the iteration form of the membrane potential is:
\begin{align}
    &H_t=\alpha_t H_{t-1}+(1-\alpha_t)X_t. \label{eq:dynamic_decay_iterative}
\end{align}
which can be rewritten as
\begin{align}
    &H_t=\sum_{i=1}^{t}\left(\prod_{j=i+1}^t{\alpha_j}\right)(1-\alpha_i)X_i.\label{eq:dynamic_decay_noniterative}
\end{align}
Note that the coefficient of $X_i$ is input-dependent, which implies that the combination of $X_i$ is not actually a linear term. From Definition \ref{non-linearity}, we conclude that dynamic decay introduces nonlinearity.

Next, we show how $\mathbf{\alpha}_t$ controls the membrane potential $H_t$. Given $\Delta \in \mathbb{N}^+$, suppose $H_{t-\Delta}\geq V_{\text{th}}$ and $X_{t-\Delta+1}, ..., X_t<V_{\text{th}}/\Delta$. Note that when
\begin{align}
    \alpha_{t-\Delta+1}<\frac{V_{\text{th}}-X_{t-\Delta+1}}{H_{t-\Delta}-X_{t-\Delta+1}} \in (0,1]\label{eq:dynamic_decay_delta_short_condition}
\end{align}
We have
\begin{equation}
\begin{aligned}
    H_{t-\Delta+1}&=\alpha_{t-\Delta+1} H_{t-\Delta}+(1-\alpha_{t-\Delta+1})X_{t-\Delta+1}\\
    &=\alpha_{t-\Delta+1}(H_{t-\Delta}-X_{t-\Delta+1})+X_{t-\Delta+1}\\
    &<V_{\text{th}}-X_{t-\Delta+1}+X_{t-\Delta+1}=V_{\text{th}}. \label{eq:proof_01}
\end{aligned}
\end{equation}
For $m = 2,3,...,\Delta$, we sequentially derive
\begin{equation}
\begin{aligned}
    H_{t-\Delta+m}&=\alpha_{t-\Delta+m} H_{t-\Delta+m-1}+(1-\alpha_{t-\Delta+m}){X}_{t-\Delta+m} \\
    &< \alpha_{t-\Delta+m} V_{\text{th}}+(1-\alpha_{t-\Delta+m})V_{\text{th}}=V_{\text{th}}
\end{aligned}
\end{equation}
When $m=\Delta$, $H_t<V_{\text{th}}$. From Definition \ref{local_control}, every $\alpha_t$ satisfying Eq. \ref{eq:dynamic_decay_delta_short_condition} can guarantee $\Delta$-short control ability and avert continuous firing when inputs are smaller than threshold.

For long control ability, suppose that $\{X_t\}$ has an upper bound $C$, i.e., $X_i\leq C, i=1,2,...,t$. It is easy to get that $H_1=(1-\alpha_1)X_1\leq C$. Besides, if $H_{t-1} \leq C$, then
\begin{equation}
\begin{aligned}
    H_{t}&=\alpha_{t} H_{t-1}+(1-\alpha_{t}){X}_{t} \leq\alpha_{t} C+(1-\alpha_{t})C=C
\end{aligned}
\end{equation}
By mathematical induction, we know that $\{H_t\}$ has an upper bound $C$. From Definition \ref{global_control}, $\alpha_t$ has long control ability. 

\end{proof}
\noindent\textit{Remark}: Dynamic decay surpasses the reset mechanism in both introducing nonlinearity and controlling the membrane potential. On one hand, the nonlinearity of dynamic decay is governed by the input-dependent coefficient $\alpha_t$, which can theoretically vary freely within [0,1]. In contrast, the nonlinearity in the reset mechanism is discrete and limited (either a complete reset to zero or a fixed subtraction). On the other hand, dynamic decay provides more flexible control over the membrane potential. In particular, regarding $\Delta$-short control, it allows a defined upper bound $\Delta$ for an arbitrarily large membrane potential, while enabling the actual duration of influence within this window to adapt dynamically based on the input. In the above proof, we derived the condition under which a large membrane potential has an effective duration of 1 (Eq. \ref{eq:dynamic_decay_delta_short_condition}). In fact, we can further generalize this result to obtain the condition for an actual duration of $\tau \in [1, \Delta]$:
\begin{equation}
    \begin{aligned}
        \begin{cases}
        \displaystyle \alpha_{t-\Delta+i}\ge \frac{V_{\text{th}}-X_{t-\Delta+i}}{H_{t-\Delta+i-1}-X_{t-\Delta+i}}, &i=1,2,...,\tau-1\\[10pt]
        \displaystyle \alpha_{t-\Delta+i}<   \frac{V_{\text{th}}-X_{t-\Delta+i}}{H_{t-\Delta+i-1}-X_{t-\Delta+i}}, &i=\tau
        \end{cases}.
    \end{aligned}
\end{equation}
In this case, $H_{t-\Delta}, H_{t-\Delta+1}, ..., H_{t-\Delta+\tau-1}\ge V_{\text{th}}$ and $H_{t-\Delta}$ triggers a total of $\tau$ spikes.

\subsection{Dynamic Decay in Integer-valued Training Case}\label{detail_integer_spike_proof}

When the integer-valued training technique is introduced, dynamic decay is still able to retain the two functions of the reset mechanism. According to Proposition 1 of \cite{yao2025scaling}, integer-value output (with upper bound $N$) is equal to the sum of spikes generated by IF-SR (IF with Soft Reset) spiking neuron with $N$ timesteps. Therefore, functions of the reset mechanism are still preserved at the single-neuron level. Consequently, Proposition \ref{equivalent_function_reset_dynamic_decay} should still hold in the integer spike scenario.

In fact, assuming the integer spiking function $f(H_t)=\lfloor\mathrm{Clip}(H_t,0,N)\rfloor$ and $V_{\mathrm{th}}=1$, where $\mathrm{Clip}(x,0,N)$ means clipping the input $x$ to the range [0, $N$], and $\lfloor.\rfloor$ is the floor function. Since the functions of non-linearity and membrane potential control in dynamic decay are independent of the choice of the spike firing function, this concludes our functional analysis in Proposition \ref{equivalent_function_reset_dynamic_decay} in integer-valued training case.

\section{Experimental Details}\label{sec: detail_experiment}

In this work, we set DSN hyperparameters $N=4$, $k=4$, and $\tau=0.25$. The dynamic decay form of DSN matches the HGRN operator in flash-linear-attention library\footnote{\url{https://github.com/fla-org/flash-linear-attention}}. Therefore, in this paper, we leverage this Triton operator to enable parallel training of DSN.

\subsection{Training Efficiency and Extrapolation}
\textbf{Training Efficiency.} We evaluate the training efficiency of different parallelizable spiking neurons by measuring the average runtime of 100 forward and backward passes under sequence lengths from 2k to 16k. We set the batch size processed by the spiking neurons to 16 and the number of spiking neurons to 512. For sliding PSN \citep{fang2023parallel}, we set the sliding window $k=64$.

\textbf{Extrapolation.} We conduct experiments on WikiText-103 \citep{merity2017pointer} dataset using SpikingSSM \citep{shen2025spikingssms} architecture with different spiking neurons. Models are trained on sequences of length 2k and evaluated on sequences of length 1k to 30k. During the testing phase, to allow the construction of sufficiently long sequences on the test set, we set the batch size to 1. The remaining experimental settings follow \ref{sec: wikitext-103}. For masked PSN \citep{fang2023parallel}, We align the order $k$ with the training length. For sliding PSN \citep{fang2023parallel}, we set the sliding window $k=32$.

\subsection{Sequential CIFAR, ImageNet and CIFAR10-DVS}\label{detail_scifar_experiment}

\textbf{Sequential CIFAR.} We use the width of the image as the sequence length $(L=32)$ to obtain a serialized version of CIFAR dataset. The model architecture is consistent with that of PSN \citep{fang2023parallel} as detailed in Table \ref{tab: scifar model configurations}. For hyperparameter settings, the training is conducted over 256 epochs with a cosine decay learning rate schedule, starting at a maximum of 0.001. We set the batch size to 128 and select AdamW optimizer \citep{loshchilov2019decoupled} with zero weight decay. 
\begin{table*}[ht]
\caption{Configurations of Conv-based SNNs for Sequential CIFAR dataset. BN: BatchNorm, FC: Fully Connected.}
\label{tab: scifar model configurations}
\vspace{4pt}
\centering
\renewcommand{\arraystretch}{1.35}
\resizebox{\linewidth}{!}{
\begin{tabular}{c|c|c}
\hline
Stage & Layer Specification & Configuration \\ \hline
\multirow{2}{*}{1} & Conv1D-BN-DSN Block $\times$ 3 & Conv: (3, stride=1, padding=1), Dim: 128 \\
 & Average Pooling & Feature size: 32 $\rightarrow$ 16 \\ \hline
\multirow{2}{*}{2} & Conv1D-BN-DSN Block $\times$ 3 & Conv: (3, stride=1, padding=1), Dim: 128 \\
 & Average Pooling & Feature size: 16 $\rightarrow$ 8 \\ \hline
3 & Flatten-FC1-DSN-FC2 & FC1: 1024 $\rightarrow$ 256, FC2: 256 $\rightarrow$ class\_num \\
\hline
\end{tabular}
}
\end{table*}

\textbf{ImageNet and CIFAR10-DVS.} For ImageNet, our experimental setup is identical to that of PSN  \citep{fang2023parallel}. For CIFAR10-DVS, we use AdamW \citep{loshchilov2019decoupled} as the optimizer with a learning rate of 0.001, while keeping all other settings consistent with PSN.

\subsection{Time-series Forcasting Tasks}\label{detail_timeseries_forecasting}
We rely on SnnTorch \citep{Eshraghian2023training} and SpikingJelly \citep{fang2023spikingjelly} to build the baseline networks. For SNNs with LIF neurons, we set the training timestep $T=4$, the threshold $V_{\text{th}}=1.0$, and the decay rate $\beta=0.99$. For SNNs with DSN neurons, thanks to integer-valued training techniques, we do not directly expand timesteps to perform 0-1 encoding for temporal tasks. Instead, $N=4$ is regarded as the expanded 4 timesteps. The architecture and size of DSN-based model are aligned with \cite{lv2024efficient}. For training hyperparameters, we use a batch size of 128 and employ Adam optimizer \citep{diederik2015adam} with a learning rate of $1 \times 10^{-4}$. An early stopping strategy is implemented with a tolerance of 30 epochs. The experiments are conducted using 4 NVIDIA RTX A6000 GPUs.

To assess the performance of our model, we use the Root Relative Squared Error (RSE) and the coefficient of determination ($\text{R}^2$), defined as follows:
\begin{align}
    \mathrm{RSE}&=\sqrt{\frac{\sum_{m=1}^M||\mathbf{Y}^m-\hat{\mathbf{Y}}^m||^2}{\sum_{m=1}^M||\mathbf{Y}^m-\bar{\mathbf{Y}}||^2}}, \\
    \mathrm{R}^2&=\frac1{MCL}\sum_{m=1}^M\sum_{c=1}^C\sum_{l=1}^L\left[1-\frac{(Y^m_{c,l}-\hat{Y}^m_{c,l})^2}{(Y^m_{c,l}-\bar{Y}_{c,l})^2}\right].
\end{align}
In these formulas, $M$ is the size of the test set, $C$ is the number of channels, and $L$ is the prediction length. $\bar{\mathbf{Y}}$ represents the average of $\mathbf{Y}^m$. $Y^m_{c,l}$ denotes the $l$-th future value of the $c$-th variable in the $m$-th sample, while $\bar{Y}_{c,l}$ is its average across all samples. $\hat{\mathbf{Y}}^m$ and $\hat{Y}_{c, l}^{m}$ denote the ground truth values. Unlike Mean Squared Error (MSE) or Mean Absolute Error (MAE), these metrics are less sensitive to the absolute scale of the dataset, making them particularly well suited for time-series forecasting tasks.

Regarding the improvements in Spikformer \citep{zhou2023spikformer}, in addition to replacing the spiking neurons, we also make architectural modifications to achieve better performance. Specifically, we expand the first DSN in the Feed-Forward Network (FFN) block to an enhanced DSN mentioned in Sec. \ref{sec: methods} to improve the interaction between different neuron channels. However, this increases the total number of parameters in the FFN block by $16C^2$, where $C$ is the number of channels. To maintain the same total parameter count $8C^2$ of the FFN block, we reduce the expansion ratio of the linear mapping from the usual 4 to 2. The architecture of the FFN block before and after modification is shown in Fig. \ref{fig: spikformer_dsn_FFNblock}.

\begin{figure}[ht]
\vskip 0.2in
\vspace{-16pt}
\begin{center}
\centerline{\includegraphics[width=0.6\columnwidth]{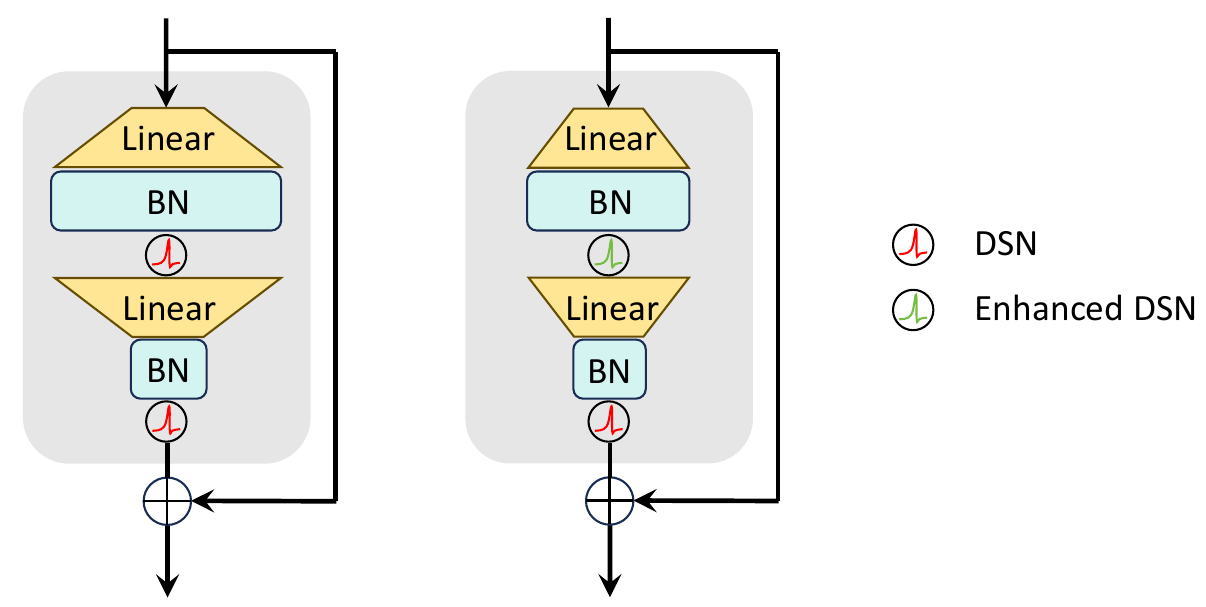}}
\vspace{-5pt}
\caption{The FFN block in Spikformer with DSN before (left) and after (right) modification. BN: BatchNorm.}
\label{fig: spikformer_dsn_FFNblock}
\end{center}
\vskip -0.2in
\vspace{-5pt}
\end{figure}

\subsection{Reinforcement Learning}\label{app:rl}

Baseline results are reproduced by strictly following the protocol of \citet{xu2025proxy}.
All results are reported as the maximum average return over five random seeds.
For fair comparison, SNN models with DSN neurons share the same architecture as the PT-LIF models in \citet{xu2025proxy}, and all experiments are conducted using the configurations released in the official implementation on Github\footnote{\url{https://github.com/xuzijie32/Proxy-Target}}.

\subsection{WikiText-103}\label{sec: wikitext-103}
Our experiment is implemented on 8 NVIDIA A800 GPUs. The network architecture and hyperparameters are largely based on S4 \citep{gu2022efficiently} and SpikingSSM \citep{shen2025spikingssms}, as shown in Table \ref{tab: wikitext configurations}. The key difference is that we shorten the length of the training text to 2048. To maintain the number of tokens per training step, we increase the batch size per GPU to 4. 
\begin{table}[ht]
\caption{Configurations of DSN-based language model on WikiText-103.}
\label{tab: wikitext configurations}
\vspace{4pt}
\centering
\renewcommand{\arraystretch}{1.15}
\begin{tabular}{c|c}
\hline
Configurations & WikiText-103  \\ \hline
Layer Depth & 16  \\
Model Dimension & 1024  \\ 
Learning Rate & 5e-4  \\
Learning Rate Schedule & Cosine Decay, with 500 warmup steps  \\
Optimizer & AdamW \citep{loshchilov2019decoupled}\\
Weight Decay & 0.01 \\
Batch Size per GPU & 8  \\ 
Epochs & 100  \\
\hline
\end{tabular}
\end{table}

\subsection{Analysis of Energy Consumption}\label{detail_energy_consumption}

We follow the method in \citep{yao2023sdsa} to evaluate the energy consumption of the Sequential CIFAR network using different spiking neurons. Specifically, the energy consumption for floating-point operations (FLOPs) is calculated by $E_{MAC} \cdot FLOPs$, while the energy consumption for spike-based operations is calculated by $E_{AC} \cdot T \cdot R \cdot FLOPs$. Here, $E_{MAC}=4.6pJ$ and $E_{AC}=0.9pJ$ in 45nm technology. $T$ denotes timestep and $R$ denotes the spike firing rate. The FLOPs of the $n$-th Conv1D layer are $k_n \cdot d_n \cdot c_{n-1} \cdot c_n$, where $k_n$ is the kernel size, $d_n$ is the sequence channel number, $c_{n-1}$ and $c_n$ are the input and output convolution channel numbers, respectively. The FLOPs of the $m$-th fully connected layer are $i_m \cdot o_m$, where $i_m$ and $o_m$ are the input and output channels of the layers. 

According to \citep{tang2025sorbet}, the Sigmoid function can be implemented using two shift operations and one addition. We scaled the energy consumption of shift operations reported in \citep{you2020shiftaddnet} proportionally to a 45nm technology, yielding an energy cost of $0.72pJ$ per shift. Therefore, the energy consumption of a single sigmoid computation can be reduced to $2*0.72+0.9=2.34pJ$.

The energy consumption from LIF neurons itself is usually considered negligible compared to that of the network architecture, including convolution and fully connected layers. In contrast, PSN and DSN have more complex internal structures, leading to non-negligible energy consumption. To ensure fairness, we present a statistical method for FLOPs within spiking neurons and summarize it in Table \ref{tab: methods of flops within spiking neuron}. 

\begin{table}[t]
\caption{Statistical methods of FLOPs within spiking neurons. $c$: number of spiking neuron. $T$: Timestep. $k$: kernel size of causal convolution.}
\label{tab: methods of flops within spiking neuron}
\vspace{4pt}
\centering
\renewcommand{\arraystretch}{1.15}
\begin{tabular}{c|c|c}
\hline
Spiking Neuron & Internal Structure & FLOPs  \\ \hline
LIF \citep{abbott1999lapicque} & Update of Membrane Potential & $c \cdot T$  \\ \hline
PSN \citep{fang2023parallel} & Update of Membrane Potential & $c \cdot T^2$  \\ \hline
sliding PSN \citep{fang2023parallel} & Update of Membrane Potential & $0.5c \cdot T^2$  \\ \hline
\multirow{3}{*}{\bf DSN (Ours)} & Causal Conv1D & $k \cdot c \cdot T$  \\
 & Sigmoid Function & $c \cdot T$  \\
 & Update of Membrane Potential & $c \cdot T$  \\
\hline
\end{tabular}
\end{table}

The spike firing rates of different layers\footnote{Notably, the input to the first convolutional layer are floating-point numbers of the original sequence, rather than processed spikes. Therefore, this layer is not involved in the calculation of the spike firing rate.} in Conv-based SNN for Sequential CIFAR using different spiking neurons are presented in Table \ref{tab: spike firing rate}. Our DSN exhibits a lower spike firing rate than that of LIF, which helps offset the additional energy cost introduced by the dynamic decay module.

\begin{table}[t]
\caption{Spike firing rates of Conv-based SNN for Sequential CIFAR10 and CIFAR100. Convx: Conv1D of the x-th layer. FC: Fully Connected.}
\label{tab: spike firing rate}
\vspace{4pt}
\centering
\renewcommand{\arraystretch}{1.15}
\resizebox{\linewidth}{!}{
\begin{tabular}{cc|ccccccc|c}
\hline
Dataset & Methods & Conv2 & Conv3 & Conv4 & Conv5 & Conv6 & FC1 & FC2 & Average\\ \hline
\multirow{4}{1.8cm}[-0.0ex]{\centering Sequential \\ CIFAR10} & LIF \citep{abbott1999lapicque} & 0.1511 & 0.1422 & 0.1811 & 0.1553 & 0.1457 & 0.0926 & 0.0647 & 0.1499  \\ 
& PSN \citep{fang2023parallel} & 0.2200 & 0.3101 & 0.1575 & 0.1542 & 0.1516 & 0.1439 & 0.1239 & 0.2143  \\ 
& sliding PSN \citep{fang2023parallel} & 0.1792 & 0.1875 & 0.1297 & 0.2538 & 0.1923 & 0.0764 & 0.1172 & 0.1820  \\ 
& \bf DSN (Ours) & 0.1349 & 0.1337 & 0.1301 & 0.1301 & 0.0982 & 0.0380 & 0.0484 & \bf 0.1238   \\ \hline
\multirow{4}{1.8cm}[-0.0ex]{\centering Sequential \\ CIFAR100} & LIF \citep{abbott1999lapicque} & 0.2264 & 0.1281 & 0.1881 & 0.1581 & 0.1561 & 0.1018 & 0.1584 & 0.1697  \\ 
& PSN \citep{fang2023parallel} & 0.3221 & 0.2127 & 0.1887 & 0.1682 & 0.1509 & 0.1735 & 0.1458 & 0.2226  \\ 
& sliding PSN \citep{fang2023parallel} & 0.1988 & 0.2042 & 0.1394 & 0.2653 & 0.1551 & 0.0827 & 0.1888 & 0.1900  \\ 
& \bf DSN (Ours) & 0.1384 & 0.1420 & 0.1404 & 0.1349 & 0.1240 & 0.0362 & 0.0973 & \bf 0.1324   \\
\hline
\end{tabular}
}
\end{table}

\subsection{Approximation Experiment}\label{detail_approximation_experiment}

Dynamic decay adaptively retains part of historical information stored in the membrane potential based on changing input. From the perspective of approximation, dynamic decay is powerful to simulate the behaviors of spiking neurons with various internal structures. During training, the spiking neuron learns to construct different reset mechanisms to model different input by regulating decay. For example, if the information is better suited to be encoded in the form of hard reset, the spiking neuron only needs to approximate a binary classifier to decide whether to set $\alpha_t$ to be constant $\beta$ or 0. This plasticity of dynamics potentially breeds rich memory abilities. To verify the expressiveness of spiking neurons with dynamic decay, we design an experiment of using dynamic decay to approximate the behaviors of multiple LIF neurons with hard or soft reset.

\textbf{Overview.} To begin with, we manually construct two distinct datasets with a timestep of $T=128$ named A and B, and split them into training and test set. Dataset A has input signals following a normal distribution with parameters $(\mu, \sigma^2)$, while dataset B is a collection of more regular signals including sine functions, sigmoid functions, step functions and Poisson encoding with different parameters. Afterwards, these signals are input into 6 LIF neurons with different reset mechanisms and membrane time constants. Then, we apply dynamic decay across $C=6$ channels with the same input signals to approximate the membrane potential with that of the LIF neurons described above, using the Mean Squared Error (MSE) loss function. Lastly, we calculate the spike firing accuracy of the test set as evaluation metric. 

\textbf{Datasets.} The normal distribution parameters of Dataset A are $\mu=1,\sigma=2$. A total of 11,000 samples are collected, with a training-to-test ratio of $10:1$. The signal generation methods of Dataset B is shown in Table \ref{tab: dataset B configurations}. Each type of signal generates 200 samples (totally 800 samples), with 10\% randomly selected as test set and the remaining samples used for training.

\begin{table}[ht]
\caption{The signal generation methods of Dataset B. $x=0,1,...,T-1.$ The notation $[a:b:c]$ means selecting $c$ evenly spaced values from $a$ to $b$. For example, $[5:15:5]$ is equal to $5,7.5,10,12.5,15$. Different parameters can be combined with each other to obtain samples with different characteristics, with the corresponding $c$ multiplied. For Poisson Encoding, $\mathrm{Random}(\cdot)$ denotes the random sampling of a floating-point number from the interval $[0,1]$, and each set of parameters is repeated 8 times to generate 8 samples.}
\label{tab: dataset B configurations}
\vspace{4pt}
\centering
\renewcommand{\arraystretch}{1.45}
\resizebox{\linewidth}{!}{
\begin{tabular}{c|c|c}
\hline
Signal Type & Formulation & Specification  \\ \hline
Sine Function & $\mathrm{input}=A\sin(\omega x)+B$ & $\omega=2\pi \frac{C-1}{T-1},A=[-2:3:5], B=[-2:3:8],C=[5:15:5]$\\ \hline
Sigmoid Function & $\mathrm{input}=A \cdot\mathrm{Sigmoid}(x')$ & $x'=\frac{20}{T-1}x-10+B,A=[-2:5:10], B=[10:10:20]$\\ \hline
Step Function & $\mathrm{input}=A \cdot\mathrm{Heaviside}(x')$ & $x'=x-B,A=[-2:5:10], B=[0:T:20]$\\ \hline
Poisson Encoding & $\mathrm{input}=A \cdot \mathrm{Heaviside}(p)$ & $p=\mathrm{Random}(x)-p_0,A=[-1:5:5], p_0=[0.3:1:5]$\\ 
\hline
\end{tabular}
}
\end{table}

\textbf{Spiking Neurons.} We set the threshold of LIF neurons to be 1, and the structure of dynamic decay is as follows:
\begin{align}
    \mathbf{X}_{t}'&=\mathrm{CausalConv1D_{up}}(\mathbf{X}_{t-k+1:t})\\
    \mathbf{X}_{t}''&=\mathrm{ReLU}(\mathbf{X}')\\
    \mathbf{X}_{t}'''&=\mathrm{CausalConv1D_{down}}(\mathbf{X}_{t-k+1:t}'')\\
    \bm{\alpha}_{t}&=\mathrm{Sigmoid}(\mathbf{X}_t''')^{1/\tau}
\end{align}
Here, $\mathbf{X}_t \in \mathbb{R}^{C\times 1}$, and $\mathbf{X}_{t-k+1:t}$ denotes $k$ inputs from $\mathbf{X}_{t-k+1}$ to $\mathbf{X}_t$. The function $\mathrm{CausalConv1D(\cdot)}$ is causal 1D convolution and the indices $\mathrm{up}$ and $\mathrm{down}$ represent the expansion of the input channels from $C$ to $eC$, and the reduction from $eC$ to $C$, respectively. $\tau$ is a hyperparameter to fine-tune $\bm{\alpha}_t$. In this experiment, we set $k=e=8$, and $\tau=0.5$.

\textbf{Training.} We set a batch size of 128 and employ Adam optimizer \citep{diederik2015adam} with a cosine decay schedule whose peak learning rate is $1 \times 10^{-2}$. The training epochs are 100. To align with the techniques used in the main text, we also conducted experiments using dynamic decay to approximate integer-valued spikes (we set $N=4$). Since integer-valued spikes are only meaningful in the case of soft reset \citep{yao2025scaling}, we fit only LIF neurons with soft reset in this case.

\textbf{Results and Discussions.} Results in Table \ref{tab: Approximation results} and Fig. \ref{fig: Approximation results} show that dynamic decay generally fits well to LIF neurons with different reset structures under various types of input signals, indicating its potential for expressiveness. When the spikes become integers, the fitting accuracy of dynamic decay further improves on both datasets, supporting our view that the integer-valued training technique and dynamic decay have a complementary effect in terms of expressiveness. Specifically, there are two details that merit our attention: 

\begin{itemize}
\item As the membrane time constant increases, the fitting accuracy declines. This could be due to the growing influence of historical information on the integration mechanism of the spiking neuron, and modeling such information has always been a challenging task. However, in the current modeling of the LIF model, the value of $\tau_m$ typically does not exceed the range specified in Table \ref{tab: Approximation results} (usually $\tau_m=2$ in \citep{yao2023sdsa, yao2025scaling}), and our focus is on more general fitting scenarios. Additionally, the introduction of integer-valued spikes can significantly suppress this fitting error.
\item When the membrane potential approaches the threshold, the error between the membrane potential predicted by dynamic decay and the true value generated by the LIF neuron is small for noise that follows a normal distribution. However, for a sine wave signal, the error between the two is larger (see Fig. \ref{fig: Approximation results}). We speculate that the cause of this difference lies in the fact that the proportion of data near the threshold is smaller for the sine signal compared to the noise signal with a mean $\mu$ equal to the threshold. This makes it more difficult for dynamic decay to learn how to handle membrane potential fluctuations near the threshold.
\end{itemize}


\begin{table}[ht]
\small
\vspace{-6pt}
\caption{Experimental results of applying dynamic decay to approximate various LIF neurons with reset mechanisms on manually constructed datasets with different signals. Each channel of the parallel spiking neuron with dynamic decay is fitted with a LIF neuron. We report the spike firing accuracy (\%) across 6 different channels and average them. *: results with integer-valued spike.
}
\label{tab: Approximation results}
\vspace{6pt}
\centering
\renewcommand{\arraystretch}{1.15}
\setlength{\tabcolsep}{3pt}
\resizebox{\linewidth}{!}{
\begin{tabular}{cccccc}
\toprule
Channel & LIF neurons to fit & Dataset A & Dataset B & Dataset A* & Dataset B*\\
\midrule
$1$ & hard reset, $\tau_m=4/3$ & 99.49 & 98.36 & --- & ---\\
$2$ & hard reset, $\tau_m=2$ & 95.10 & 95.50 & --- & ---\\
$3$ & hard reset, $\tau_m=4$ & 85.87 & 90.18 & --- & ---\\
$4$ & soft reset, $\tau_m=4/3$ & 99.03 & 98.20 & 99.30 & 99.01\\
$5$ & soft reset, $\tau_m=2$ & 93.87 & 96.40 & 98.50 & 98.14\\
$6$ & soft reset, $\tau_m=4$ & 84.83 & 91.65 & 97.59 & 97.82\\ \midrule
Average & --- & 92.97 & 95.05 & 98.46 & 98.32\\
\bottomrule
\end{tabular}
}
\vspace{-8pt}
\end{table}

\begin{figure}[ht]
\vskip 0.2in
\vspace{-8pt}
\begin{center}
\centerline{\includegraphics[width=0.6\columnwidth]{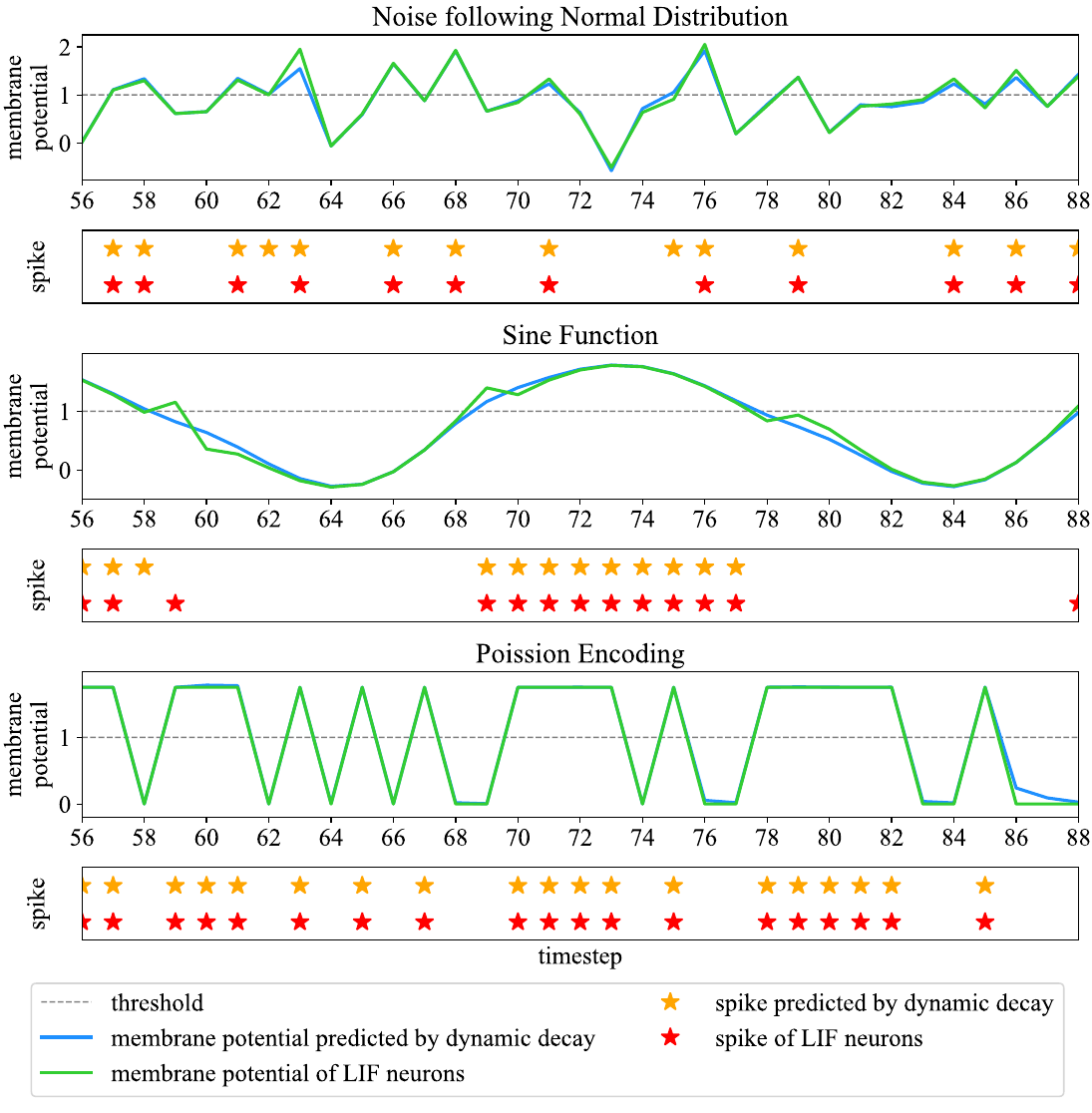}}
\vspace{-5pt}
\caption{Signal responses including membrane potential and spike for LIF neuron and its dynamic decay prediction on channel 2. Subplots from top to bottom depict the responses to a noise following normal distribution, sine function, and Poisson encoding.}
\label{fig: Approximation results}
\end{center}
\vskip -0.2in
\vspace{-5pt}
\end{figure}

\end{document}